\newcommand{\commentt}[2]{#1}
\newcommand{\comment}[1]{}
\newcommand{\cov}{{\rm rCov}}
\newcommand{\var}{{\rm Var}}
\newcommand{\tr}{{\rm tr}}
\newcommand{\std}{{\rm Std}}
\newcommand{\titt}{Unbiased least squares regression via averaged stochastic gradient descent}
\commentt{
\newcommand{\citet}{\citeasnoun}
\documentclass[a4paper,11pt]{article}
\usepackage{graphicx,slashbox,amsthm}
\usepackage[full]{harvard}

\usepackage[ps,dvips]{xy}
\title{\titt}
\author{Nabil Kahal\'e
\thanks{\emph{ESCP Business School, Paris, France; {e-mail: }{nkahale@escp.eu}.}}}
\usepackage{amstext}
\usepackage{amssymb}
\usepackage{amsmath}
\usepackage{setspace}
\usepackage[margin=1in]{geometry}
\usepackage[english]{babel}
\bibliographystyle{dcu}
\usepackage{color}
\usepackage{textcomp}
\usepackage{varioref}
\usepackage{hyperref}
\usepackage{algpseudocode}
\usepackage{algorithm}

\usepackage{xspace}
\usepackage{pgfplots}
\usepackage{caption,subcaption}

\begin{document}

\newtheorem{example}{Example}[section]
\newtheorem{theorem}{Theorem}[section]
\newtheorem{conjecture}{Conjecture}[section]
\newtheorem{lemma}{Lemma}[section]
\newtheorem{proposition}{Proposition}[section]
\newtheorem{remark}{Remark}[section]
\newtheorem{corollary}{Corollary}[section]
\newtheorem{definition}{Definition}[section]
\numberwithin{equation}{section}
\maketitle
\newcommand{\ABSTRACT}[1]{\begin{abstract}#1\end{abstract}}
\newcommand{\citep}{\cite}
}
{
\documentclass[opre,nonblindrev]{informs3} 
\usepackage[margin=1in]{geometry}
\usepackage{textcomp}
\usepackage{pgfplots}
\DoubleSpacedXII



\usepackage[round]{natbib}
 \bibpunct[, ]{(}{)}{,}{a}{}{,}%
 \def\bibfont{\small}%
 \def\bibsep{\smallskipamount}%
 \def\bibhang{24pt}%
 \def\newblock{\ }%
 \def\BIBand{and}%

\TheoremsNumberedThrough     
\ECRepeatTheorems

\EquationsNumberedThrough    

\MANUSCRIPTNO{} 

\renewcommand{\qed}{\halmos}

\begin{document}
\RUNAUTHOR{Kahal\'e}
\RUNTITLE{Multilevel  methods for discrete Asian options}
\TITLE{General multilevel Monte Carlo methods for pricing discretely monitored
Asian options}
\ARTICLEAUTHORS{%
\AUTHOR{Nabil Kahal\'e}

\AFF{ESCP Business School, Paris, France, \EMAIL{nkahale@escp.eu} \URL{}}
} 
\bibliographystyle{informs2014} 
}

\ABSTRACT{
\comment{We consider an on-line  least squares regression problem with  optimal solution \(\theta^{*}\) and Hessian matrix \(H\), and study a time-average stochastic gradient descent estimator \(\bar\theta_{k}\) of \(\theta^{*}\). Using a randomized multilevel Monte Carlo approach, we provide  an unbiased estimator of the negated bias of  \(\bar\theta_{k}\),  and   an unbiased estimator \(\hat f_{k}\) of  \(\theta^{*}\). The expected number of time-steps of  \(\hat f_{k}\)  is of order \(k\), with \(O(1/k)\) expected excess risk,  where the constant behind the \(O\) notation depends on parameters of the regression and is a poly-logarithmic function of the smallest  eigenvalue of \(H\). We provide both a biased and an unbiased estimator of the  expected excess risk   of  \(\bar\theta_{k}\) and of \(\hat f_{k}\) that do not require knowledge of either \(H\) or  \(\theta^{*}\). We describe an ``average-start'' version of our estimators with similar properties. Our numerical experiments confirm our theoretical findings.}

 We consider an on-line  least squares regression problem with  optimal solution \(\theta^{*}\) and Hessian matrix \(H\), and study a time-average stochastic gradient descent estimator of \(\theta^{*}\). For \(k\ge2\), we provide  an unbiased estimator of  \(\theta^{*}\) that is a modification of the time-average estimator,  runs with an expected number of time-steps of  order \(k\), with \(O(1/k)\) expected excess risk. The constant behind the \(O\) notation depends on parameters of the regression and is a poly-logarithmic function of the smallest  eigenvalue of \(H\). We provide both a biased and unbiased estimator of the  expected excess risk   of the time-average estimator and of its unbiased counterpart, without requiring knowledge of either \(H\) or  \(\theta^{*}\). We describe an ``average-start'' version of our estimators with similar properties. Our approach is based on randomized multilevel Monte Carlo. Our numerical experiments confirm our theoretical findings.    }
\commentt{
}{\maketitle
Subject classifications: Finance: asset pricing; Simulation: Efficiency ; Analysis of algorithms: Computational complexity}
\section{Introduction}
Monte Carlo methods are used in a variety of  fields such as financial engineering\comment{~\cite{glasserman2004Monte}}, queuing  networks,  health-care and machine learning\comment{~\cite{russo2014learning,aravkin2020trimmed}}. A drawback of  these methods is their high computational cost. The efficiency of  Monte Carlo methods can however  be improved via variance reduction techniques      \cite{glasserman2004Monte,asmussenGlynn2007,gower2020variance}. 
 In particular, the  multilevel  Monte Carlo 
 method (MLMC) in  \citet{Giles2008}     dramatically reduces the computational cost of estimating functionals of stochastic differential equations. \citet{mcleish2011}, \citet{glynn2014exact} and    \citet{GlynnRhee2015unbiased}    give related randomized multilevel Monte Carlo methods (RMLMC) that calculate  unbiased estimators of   functionals of Markov chains and of stochastic differential equations. RMLMC  and related techniques  have been used in a variety of  contexts such as the unbiased estimation for steady-state expectations  of a Markov chain \cite{glynn2014exact,kahale2023}, the design of Markov chain Monte Carlo methods~\cite{agapiou2018unbiased,middleton2018unbiased,wang2023unbiased}, statistical inference~\cite{vihola2021unbiasedInference,heng2024unbiased}, estimating the expected cumulative discounted cost~\cite{cui2020optimalDisc}, and pricing of Asian options under general models~\cite{kahale2020Asian}.  \citet{Vihola2018} describes  stratified  RMLMC methods that, under certain conditions, are asymptotically as
efficient as Multilevel  Monte Carlo.
Unbiased estimators have two main advantages.
First, drawing independent copies of an unbiased estimator allows the construction of normal confidence intervals \cite[Section III.1]{asmussenGlynn2007}.  Second, this construction can be parallelized efficiently.

This paper applies the RMLMC approach to least squares regressions in an on-line setting inspired from  
 \citet{bachMoulines2013non}. Least squares regressions are a fundamental tool in statistics and machine learning, and their analysis leads to a better understanding of more sophisticated models such as logistic regressions and neural networks. Our approach is based on the  stochastic gradient descent (SGD) algorithm.  SGD and its variants are commonly used in large-scale machine learning optimization \cite{bottou2018optimization}.  It is assumed that \((x,y)\) is a square-integrable random vector in \(\mathbb{R}^{ d}\times\mathbb{R}\), where  \(\mathbb{R}^{ d}\) denotes the set of \(d\)-dimensional column vectors, and that \(((x_{t},y_{t}), t\geq0)\) is a sequence of independent copies of  \((x,y)\).  For a \(d\)-dimensional column vector \(\theta\), set \(L(\theta):=\frac{1}{2}E((y-x^{T}\theta)^{2})\). It is well-known that the function \(L\) is convex and attains its minimum at a  \(d\)-dimensional column vector  \(\theta^{*}\) that is  solution to the equation
\begin{equation}\label{eq:NormalEq}
H\theta^{*}=E(y{x}),
\end{equation}where \(H:=E({x}x^{T})\) is the Hessian matrix.

A standard SGD algorithm for minimizing \(L\) uses the random sequence of \(d\)-dimensional column vectors \((\theta_{t},t\geq0)\) satisfying the recursion \begin{equation}\label{eq:BasicSGD}
\theta _{t+1}=\theta _{t}-\gamma (x_{t}^T\theta_{t}-y_{t})x_{t},
\end{equation} where  \(\theta_{0}\) is a deterministic  \(d\)-dimensional column vector and  \(\gamma\)  is a constant step-size.    An advantage of \eqref{eq:BasicSGD} is that it can be performed with \(O(d)\) arithmetic operations. For \(k\ge1\), define the time-average (or tail-average) estimator  
\begin{equation*}
\bar \theta _{k}:=\frac{1}{k-s(k)}\sum ^{k-1}_{t=s(k)}\theta_{t},
\end{equation*}
where  \(s(k)\in[0,k-1]\) is a burn-in period that may depend on \(k\). Under certain assumptions, 
 \citet{bachMoulines2013non} show an upper-bound of  \(O(1/k)\) on the expected excess risk \(E(L(\bar \theta _{k}))-L(\theta^{*})\). It is noteworthy that the constant behind their \(O\)-notation does not depend on the smallest eigenvalue \(\lambda_{\min}\) of \(H\), which can be very small. We are not aware of a better non-asymptotic bound on the expected excess risk of \(\bar \theta _{k}\)   that does not depend on  \(\lambda_{\min}\).  A similar bound is shown in \citet[Corollary 4]{BachDieuleveut2016}.    
 \citet{jain2018parallelizing} study mini-batching and tail-averaging SGD algorithms for least squares
regressions in an on-line strongly-convex setting, and provide a sharp upper bound on the expected excess risk of their estimators.
 \citet{BachLeastSquaresJMLR2017} analyse an averaged accelerated SGD  for on-line least squares regressions, using a stochastic gradient recursion that depends explicitly on \(H\).
They provide a bound of \(O(1/k^{2})\) on forgetting the initial conditions, and a bound of \(O(d/k)\)  in terms of dependence on the noise, after \(k\) time-steps.
  \citet{jain2018accelerating} analyse another averaged accelerated SGD algorithm for on-line least squares regressions. They show that the initial error decays  at a geometric rate that improves upon that of averaged SGD.      
\citet{sid2021stochastic} design a MLMC low-bias estimator of the minimizer of Lipschitz strongly-convex functions based on stochastic gradients.
 \citet{han2024online} combine  SGD and lasso  techniques to compute asymptotically normal  estimators in  online homoscedastic least squares regressions with high-dimensional data. By definition, the error term \(y-x^{T}\theta^{*}\) is homoskedastic if it is  independent of the covariate \(x\). \comment{For any fixed \(k\), however, their bound is of order  \( {\lambda_{\min}}^{-2}\), where  \( \lambda_{\min}\)  is the smallest eigenvalue of \(H\).} 
\subsection{Contributions}
This paper constructs unbiased RMLMC estimators of the negated bias \(\theta^{*}-E(\bar\theta_k)\). More precisely, for any integer \(k\geq2\), we build a square-integrable \(d\)-dimensional column vector
 \(Z_{k}\) with  \(E(Z_{k})=\theta^{*}-E(\bar\theta_k)\). 
This allows the construction of a square-integrable \(d\)-dimensional column vector
 \(\hat f_{k}\)  with    \(E(\hat f_{k})=\theta^{*}\). The expected number of copies of \((x,y)\)  required to sample \(Z_{k}\) and \(\hat f_{k}\) is of order \(k\). We show a non-asymptotic upper bound of  \(O(1/k)\)  on the expected excess risk of \(\hat f_{k}\). Our bound on the expected excess risk of \(\hat f_{k}\) is no worse than  the bound of  
 \citet{bachMoulines2013non} on the expected excess risk of \(\bar\theta_k\), up to a poly-logarithmic factor independent of \(k\). We are not aware of a previous unbiased estimator of \(\theta^{*}\) that achieves such a tradeoff between the expected excess risk and expected time-steps. We also build an average-start version of \(Z_{k}\) and of 
 \(\hat f_{k}\)  with similar properties. 

Our analysis of the  expected excess risk uses a bias-variance decomposition inspired from the statistics and financial engineering literature. We provide both biased and unbiased estimators of the squared bias of   \(\bar \theta _{k}\), and unbiased estimators of the variance of   \(\bar \theta _{k}\) and of \(\hat f_{k}\). These estimators do not require knowledge of either \(H\) or  \(\theta^{*}\).    This allows the construction of both biased and unbiased estimators of the expected excess risk of  \(\bar \theta _{k}\) and of \(\hat f_{k}\), and of  the average of \(M\) independent copies of   \(\bar \theta _{k}\) and of \(\hat f_{k}\), for any integer \(M>1\).   As a byproduct of our techniques, we  show an upper bound on the expected excess risk of \(\bar \theta_{k}\)  that is the same as in  
\citet{bachMoulines2013non}, up to a constant factor, but  allows step-sizes to be twice as large. We provide an example where even larger step-sizes result in a divergent expected excess risk of \(\bar \theta_{k}\), as \(k\) goes to infinity. 

Except in specific examples, our results do not assume homoscedasticity. Moreover, the expected excess risk of \(\hat f_{k}\) has a   poly-logarithmic dependence on the smallest eigenvalue of \(H\). In contrast, convergence properties of many SGD-like algorithms  for general \(\mu\)-strongly convex functions have a strong dependence on \(\mu\). For instance, 
it takes  \(O(\mu^{-2}\epsilon^{-2})\) time-steps, up to a logarithmic factor, of SGD    with a constant step-size to ensure that the expected squared distance from the optimum is at most \(\epsilon^{2}\), where \(\epsilon>0\)
 \cite[Corollary 2.2]{needell2016stochastic}.  Likewise, the number of time-steps needed to achieve an expected excess risk of at most \(\epsilon^{2}\)  in SGD with a diminishing step-size is  \(O(\mu^{-2}\epsilon^{-2})\) \cite[Eq. 4.24]{bottou2018optimization}. A similar result, up to a logarithmic factor, follows from \citet[Corollary 1]{hu2024convergence} for SDG  with constant step-size based on unbiased stochastic gradients.  Finally, the expected number of time-steps needed to attain a certain bias and variance in the MLMC minimizer in 
\citet[Theorem 1]{sid2021stochastic} is proportional to \(\mu^{-2}\), up to a poly-logarithmic factor. 
\subsection{Other related work}
An exact solution to the least squares regression problem with \(n\) constraints in dimension \(d\), where \(n\geq d\), can be found in \(O(nd^2)\) arithmetic operations \cite{golub2013matrix}. Our  methods are closely related to the algorithms  in \citet{glynn2014exact} and \citet{kahale2023} for estimating steady-state expectations for  Markov chains. Like their methods, our construction uses RMLMC and coupling from the past.  However,  \citet{glynn2014exact} do not use time-averaging, while the randomized multilevel estimator in  \citet{kahale2023} uses tail-averaging within each level. In contrast,    our approach uses a random or an average start within each level.  In related work,  \citet{blanchet2019unbiased} construct a RMLMC unbiased estimator of the optimal solution  of general constrained stochastic convex optimization problems. Their construction is based on sample average approximations, which necessitates the computation of the optimal solution in finite samples, and does not rely on SGD or on coupling from the past. While the three aforementioned papers address steady-state estimation and optimization  for general Markov chains or functions,   our results are based on  analysing  convergence and mixing properties of SGD sequences that arise in least squares regressions. Least squares regressions remain the subject of active research. For instance, minimax bounds  on the expected excess risk of least squares regressions have been recently established under various settings. See \citet{mourtada2022exact} and references therein.

The rest of the paper is organized as follows. Notation and background for our analysis are presented in Section \ref{se:notation}. Section \ref{se:UnbiasedEstimators} presents our unbiased estimators of  \(\theta^{*}-E(\bar\theta_k)\) and of   \(\theta^{*}\).
Section \ref{se:BiasVarEst} constructs our estimators for the squared bias of   \(\bar \theta _{k}\) and the variance of   \(\bar \theta _{k}\) and of \(\hat f_{k}\). Section \ref{se:numer} presents numerical experiments that confirm our theoretical findings. Section \ref{se:conc} contains concluding remarks.  Omitted proofs are in the appendix.

\section{Notation and Background}\label{se:notation}
We first provide some notation then describe the bias-variance decomposition in Section \ref{sub:BiasVariance}. Section \ref{sub:mainAssumption} presents the main assumption, taken from \cite{bachMoulines2013non}, on the distribution of \((x,y)\).
Section \ref{sub:timeAverageEst} gives a relaxed version of the upper bound of  \citet{bachMoulines2013non} on the expected excess risk of \(\bar\theta_{k}\), that allows a burn-in period and step-sizes twice as large. Section \ref{sub:example} gives an example where even larger step-sizes implies a divergence of the expected excess risk of \(\bar\theta_{k}\). 
 
Denote by \(||v||\) the Euclidean norm of a  \(d\)-dimensional column vector \(v\). By definition, for $d\times d$ symmetric matrices $A$ and $B$, we say that \(A\preccurlyeq B\) if \(B-A\)
is positive semi-definite.  For a square-integrable \(d\)-dimensional column vector \(U\), set \(\var(U):=E(||U||^{2})-||E(U)||^{2}\) and \(\std(U):=\sqrt{\var(U)}\). Thus, \(\var(U)\) is equal to the sum of the variances of each component of \(U\). Given a positive semi-definite \(d\times d\) matrix \(\Sigma\), let \(||v||_{\Sigma}:=||\sqrt{\Sigma}v||\) for  \(v\in\mathbb{R}^{d}\), and let \(\var_{\Sigma}(U):=\var(\sqrt{\Sigma}U)\) and \(\std_{\Sigma}(U):=\std(\sqrt{\Sigma}U)\) for a square-integrable \(d\)-dimensional column vector \(U\). Thus \(||v||^{2}_{\Sigma}=v^{T}\Sigma v\) and\begin{equation}\label{eq:BiasVarSigma}
E(||U||_{\Sigma}^{2})=||E(U)||_{\Sigma}^{2}+\var_{\Sigma}(U).
\end{equation} 
\subsection{Bias and Variance}\label{sub:BiasVariance}   
A standard calculation shows that, for any \(d\)-dimensional column vector \(\theta\), we have
\begin{equation}\label{eq:minimumf}
L(\theta)-L(\theta ^{*})=\frac{1}{2}\,||\theta- \theta ^{*}||^{2}_H.
\end{equation}Let  \(\psi\)   be a square-integrable \(d\)-dimensional column vector  that is a (possibly biased)  estimator of \(\theta ^{*}\).  We have \begin{eqnarray*}
2(E(L(\psi))-L( \theta ^{*}))&=&
E(||\psi- \theta ^{*}||^{2}_{H})\\
&=&||E(\psi- \theta ^{*})||^{2}_{H}+\var_{H}(\psi- \theta ^{*}),
\end{eqnarray*}
where the first equation follows from \eqref{eq:minimumf}, and the second one from \eqref{eq:BiasVarSigma}.  
 Hence, \begin{equation}\label{eq:BiasVarianceOptGap}
2(E(L(\psi))-L( \theta ^{*}))=||E(\psi)- \theta ^{*}||^{2}_{H}+\var_{H}(\psi).
\end{equation}In other words, the quantity \(2(E(L(\psi))-L( \theta ^{*}))\)  is the sum of two terms. We refer to the first term  \(||E(\psi)- \theta ^{*}||^{2}_{H}\) as the squared bias of \(\psi\), and to the second term \(\var_{H}(\psi)\)  as the variance of \(\psi\).  Our definition of squared bias and variance is inspired from the statistics and financial engineering literature \cite{glasserman2004Monte} and is applicable to any square-integrable estimator  \(\psi\). These notions are similar in spirit but different from the bias and variance of \(\bar\theta_k\) as defined in \citet{bachMoulines2013non} and   
 \citet{jain2018parallelizing}.   

Consider now \(M\) independent copies  \(\psi^{(1)},\dots,\psi^{(M)}\)  of \(\psi\), where   \(M>1\)  is a fixed  integer, and let\begin{equation}\label{eq:PsiBarDef}
\bar\psi_{M}=\frac{1}{M}\sum^{M}_{i=1}\psi^{(i)}.\end{equation} By standard calculations, we have \(E(\bar\psi_{M})=E(\psi)\) and \(\var_{H}(\bar\psi_{M})=M^{-1}\var_{H}(\psi)\). Thus, taking the average of \(M\) independent  copies of \(\psi\) does not alter the bias but reduces the variance  by a factor of \(M\). Together with \eqref{eq:BiasVarianceOptGap},  it follows that \begin{equation}\label{eq:ExpectedOptGapM}
2(E(L(\bar\psi_{M}))-L( \theta ^{*}))= ||E(\psi)- \theta ^{*}||^{2}_{H}+\frac{\var_{H}(\psi)}{M}.
\end{equation}  

Assume now that   \(\psi\) is an unbiased estimator of \(\theta^{*}\), i.e.,   \(E(\psi)= \theta ^{*}\). Combining  \eqref{eq:BiasVarianceOptGap} and \eqref{eq:ExpectedOptGapM} shows that\begin{equation}\label{eq:excessRiskAverageUnbiased}
E(L(\bar\psi_{M}))-L( \theta ^{*})=\frac{E(L(\psi))-L( \theta ^{*})}{M}.
\end{equation} Equivalently,  taking the average of \(M\) independent  copies of \(\psi\) reduces the expected excess risk by a factor of \(M\). Let us further assume that \(\psi\) has finite expected running time \(\tau_{\psi}\), and
 let \(\epsilon>0\). By \eqref{eq:BiasVarianceOptGap} and \eqref{eq:excessRiskAverageUnbiased}, we need to sample \(M=\lceil\var_{H}(\psi)/(2\epsilon^{2})\rceil\) independent copies of  \(\psi\) in order to ensure that \(E(L(\bar\psi_{M}))-L( \theta ^{*})\le\epsilon^{2}\). Ignoring integrality constraints, the expected running time of these \(M\) copies is \( \tau_{\psi}\var_{H}(\psi)/(2\epsilon^{2})\). Thus, we can use \(\tau_{\psi}\var_{H}(\psi)\) to measure the efficiency of  \(\psi\): a small time variance product corresponds to a greater efficiency. In a related context, \citet{glynn1992asymptotic} show that the efficiency of an unbiased estimator is inversely proportional to the time variance product. \subsection{Main Assumption}\label{sub:mainAssumption}
Following   
 \citet{bachMoulines2013non}, we make the following assumption for the rest of the paper.\begin{description}
\item[Assumption 1.]    There are positive real numbers   \(\sigma\) and \(R\)   such that
\begin{equation}\label{eq:defSigma}
 E((y-x^{T}\theta^{*})^{2}{x}x^{T})\preccurlyeq \sigma^{2}H,
\end{equation}
 and\begin{equation}\label{eq:Rdef}
 E(||x||^{2}{x}x^{T})\preccurlyeq R^{2}H.
\end{equation}
\end{description}
 It can be shown that  \eqref{eq:defSigma} and  \eqref{eq:Rdef} hold under various settings. For instance, if the error term \(y-x^{T}\theta^{*}\) is independent of \(x\) and  has standard deviation \(\sigma\), then \eqref{eq:defSigma} clearly holds. In the same vein,     \eqref{eq:Rdef} holds if \(x\) is almost surely bounded.  More precisely, as observed in  \citet{bachMoulines2013non}, if \(||x||^{2}\leq R^{2}\) almost surely for a real number \(R\), then   \eqref{eq:Rdef} holds. 

Using \eqref{eq:defSigma} and  \eqref{eq:Rdef}, a standard calculation shows that \(yx\) and \(\theta_{t}\) are square-integrable for any \(t\geq0\).  \citet{bachMoulines2013non} show that, when   \(s(k)=0\) and  \(0<\gamma<1/R^{2}\), we have\begin{equation}\label{eq:BachMoulines}
E(L(\bar \theta _{k}))-L(\theta^{*})\leq\frac{1}{2k}\left(\frac{\sigma \sqrt{d}}{1-\sqrt{\gamma R^{2}}}+\frac{||\theta _{0}-\theta ^{*}||}{\sqrt{\gamma}}\right)^{2}.
\end{equation} 
\citet{BachLeastSquaresJMLR2017} show that \eqref{eq:Rdef} implies that \(\tr(H)\leq R^{2}\). As \(H\) is positive semi-definite, it follows that \begin{equation}\label{eq:HR2I}
\lambda_{\max}\le R^{2},
\end{equation}where \(\lambda_{\max}\) is the largest eigenvalue of \(H\). \subsection{Time-average estimators}\label{sub:timeAverageEst}
The bound~\eqref{eq:BachMoulines} of   \citet{bachMoulines2013non} holds when   \(s(k)=0\) and  \(0<\gamma<1/R^{2}\). Lemma \ref{le:BiasedlinReg} below relaxes these assumptions and applies to any   \(s(k)\in[0,k-1]\) and any \(\gamma\in(0,2/R^{2})\). The relaxed assumptions are in-line with the results in subsequent sections such as Theorems \ref{th:BoundhHatfk} and \ref{th:BoundhHatfkAvg}. Certain results in   Lemma \ref{le:BiasedlinReg} such as \eqref{eq:UpperBoundVarHBarTheta} and certain steps in its proof are also used to establish   Theorems \ref{th:BoundhHatfk} and \ref{th:BoundhHatfkAvg}.    Lemma \ref{le:BiasedlinReg} provides upper bounds on the squared bias, the variance and the expected excess risk of \(\bar \theta_{k}\).     
       
\begin{lemma}\label{le:BiasedlinReg} 
Assume that \((x,y)\) is square-integrable, that  \(((x_{t},y_{t}), t\geq0)\) is a sequence of independent copies of  \((x,y)\), that Assumption 1 holds,  and that \begin{equation}\label{eq:upperBoundGamma}
0<\gamma<\frac{2}{R^{2}}.
\end{equation} Then\begin{equation}\label{eq:upperBoundGammaLamdaMax}
\gamma<\frac{2}{\lambda_{\max}},
\end{equation} and, for \(k\geq1\), we have
\begin{equation}\label{eq:UpperBoundSqrBiasHBarTheta}
||E(\bar \theta _{k}-\theta ^{*})||^{2}_{H}\leq\frac{||\theta_{0}-\theta ^{*}||^{2}}{\gamma(2-\gamma \lambda_{\max})( k-s(k))},
\end{equation}and \begin{equation}\label{eq:UpperBoundVarHBarTheta}
\var_{H}(\bar \theta _{k})\le\frac{8}{k-s(k)}\left(\frac{\sigma^{2}d}{2-\gamma R^{2}}+ \frac{||\theta_{0}-\theta ^{*}||^{2}}{\gamma}\right).
\end{equation} Moreover,\begin{equation}\label{eq:thMainPositiveSemiDefinite}
E(L(\bar \theta _{k}))-L(\theta^{*})\leq\frac{1}{2(2 -\gamma R^{2})(k-s(k))}  \left( 8\sigma^{2}d+\frac{17||\theta_{0}-\theta ^{*}||^{2}}{\gamma}\right).
\end{equation}
If \(H\) is positive-definite  then \(E(\theta_{k})\rightarrow \theta^{*}\)  as \(k\) goes to infinity.
\end{lemma}
In view of \eqref{eq:upperBoundGammaLamdaMax}, the right-hand side of \eqref{eq:UpperBoundSqrBiasHBarTheta} is well-defined. Note that the bounds in \eqref{eq:UpperBoundSqrBiasHBarTheta},   \eqref{eq:UpperBoundVarHBarTheta} and      \eqref{eq:thMainPositiveSemiDefinite} do not depend on the smallest eigenvalue of \(H\), and hold even if \(H\) is singular.  Furthermore, the  upper-bound on the squared bias in \eqref{eq:UpperBoundSqrBiasHBarTheta}  does not depend on \(\sigma\). The proof of  \eqref{eq:UpperBoundSqrBiasHBarTheta} actually shows  that  \(E(\bar \theta _{k}-\theta ^{*})\) is a deterministic function  of \(\gamma\), of \(H\), and of \(\theta_{0}-\theta ^{*}\),  for any fixed \(k\geq1\). Due to the stochastic nature of \((\theta_{t},t\ge0)\), the right-hand side of   \eqref{eq:UpperBoundVarHBarTheta} is in general strictly positive even if \(\sigma=0\). When   \(s(k)=0\), it is easy to check that   \eqref{eq:thMainPositiveSemiDefinite} is no worse  than \eqref{eq:BachMoulines}, up to an absolute multiplicative constant, for any  \(\gamma\in(0,1/R^{2})\). 

In the special case where  \(\gamma=1/R^{2}\), combining \eqref{eq:HR2I} with \eqref{eq:UpperBoundSqrBiasHBarTheta} 
yields
\begin{equation*}
\, ||E(\bar \theta _{k}-\theta ^{*})||^{2}_{H}\leq\frac{R^{2}||\theta_{0}-\theta ^{*}||^{2}}{ k-s(k)},
\end{equation*} 
while \eqref{eq:UpperBoundVarHBarTheta} becomes\begin{equation*}
\var_{H}(\bar \theta _{k})\le8\frac{\sigma^{2}d+ R^{2}||\theta_{0}-\theta ^{*}||^{2}}{k-s(k)}.
\end{equation*}
Together with \eqref{eq:BiasVarianceOptGap},  it follows that \begin{equation}\label{eq:GammaR21}
E(L(\bar \theta _{k}))-L(\theta^{*})\leq\frac{8\sigma^{2}d+9R^{2}||\theta_{0}-\theta ^{*}||^{2}}{2(k-s(k))},
\end{equation}
which is slightly better than  \eqref{eq:thMainPositiveSemiDefinite} when  \(\gamma=1/R^{2}\).
\subsection{Example}\label{sub:example}
We now provide an example showing that the constant \(2\) in the right-hand side of \eqref{eq:upperBoundGamma} cannot be improved. Fix \(\sigma>0\) and   \(\gamma>2\), and let \(v\) and \(\theta_{0}\) be  \(d\)-dimensional column vectors such that \(||v||=1\) and \(v^T\theta_{0}\neq0\). Let \(D\) be a Bernoulli random variable with \(\Pr(D=0)=p\), where   \(p=2\gamma^{-1}\), and let \(U\) be  a random \(d\)-dimensional column vector uniformly on the unit sphere in \(\mathbb{R}^{d}\).  Let \(y\) be a centered square-integrable random variable with variance \(\sigma^{2}\).  Assume that \(D\), \(U\) and \(y\) are independent. Set \(x=DU+(1-D)v\), and \(s(k)=0\) for \(k\geq1\).  Thus \(x=v\) with probability \(p\), and \(x=U\) with probability \(1-p\).  Define the random sequence \((\theta_{t},t\geq0)\) via \eqref{eq:BasicSGD}, with initial state   \(\theta_{0}\)  and time-step \(\gamma\). 

 \begin{proposition}\label{pr:OptimalRadius}The matrix \(H\) is positive-definite, the function \(L(\theta)\) attains its minimum at \(\theta^{*}=0\), and the pair \((x,y)\) satisfies  \eqref{eq:defSigma} and  \eqref{eq:Rdef} with \(R^{2}=1\). Moreover, \(E(L(\bar \theta_{k}))\) goes to infinity as \(k\) goes to infinity.\end{proposition}
Thus, for any \(\gamma_{0}>2\), if we set \(\gamma=(1+\gamma_{0})/2\), the above example satisfies the conditions of Lemma \ref{le:BiasedlinReg}, except that \eqref{eq:upperBoundGamma} is replaced with the condition \(0<\gamma<\gamma_{0}/R^{2}\). 

We stress that, for a given random pair \((x,y)\) satisfying  \eqref{eq:defSigma} and  \eqref{eq:Rdef}, there may exist   \(\gamma>2/R^{2}\) such that \(E(L(\bar \theta _{k}))\) converges to \(L(\theta^{*})\) as \(k\) goes to infinity. Further results about the range of \(\gamma\) and the convergence properties of  \(E(L(\bar \theta _{k}))\) can be found in \citet{BachDefossezAIStat2015averaged} and    
 \citet{jain2018parallelizing}. \section{Unbiased time-average estimators}\label{se:UnbiasedEstimators}
This section constructs our main estimators. Section \ref{sub:singleTermEstimator} recalls the \emph{single term estimator},  a RMLMC estimator      introduced by \citet{GlynnRhee2015unbiased}. Section \ref{sub:unbiasedLongRun}   constructs the unbiased estimator \(Z_{k}\) of  \(\theta^{*}-E(\bar\theta_{k})\), based on a single term estimator, and uses  \(Z_{k}\) to construct the unbiased estimator  \(\hat f_{k}\) of  \(\theta^{*}\). Section \ref{sub:avgStart} constructs an average-start version of   \(Z_{k}\) and its corresponding unbiased estimator of  \(\theta^{*}\).       
\subsection{The single term estimator}\label{sub:singleTermEstimator} Let   \((Y_{l}, l\geq0)\) be a sequence of   \(d\)-dimensional column vectors  such that \(E(Y_{l})\) exists for \(l\geq0\) and has a limit \(\mu_{Y}\) as \(l\) goes to infinity.     Consider a probability distribution \((p_{l},l\geq0)\)  such that  \(p_{l}>0\) for \(l\geq0\). Let   \(N\in\mathbb{N}\)  be an integral  random variable independent of  \((Y_{l}, l\geq0)\)  such that \(\Pr(N=l)=p_{l}\)   for \(l\geq0\).  Theorem~\ref{th:Glynn}    describes the single term estimator \(Z\) and shows that, under suitable conditions, it is square-integrable and has expectation equal to  \(\mu_{Y}\). Theorem~\ref{th:Glynn} was shown by \citet{GlynnRhee2015unbiased}  (see also~\cite[Theorem 2]{Vihola2018}) when \(d=1\). By considering each component of the \(Y_{l}\)'s separately, it is straightforward to prove Theorem~\ref{th:Glynn}   for any finite \(d\).  
\begin{theorem}[\citet{GlynnRhee2015unbiased}]\label{th:Glynn}    Assume that \(E(||Y_{l}||^{2})\) is finite for any \(l\geq0\). Set \(Z:=(Y_{N}-Y_{N-1})/p_{N}\), with  \(Y_{-1}:=0\). If \(\sum^{\infty}_{l=0}E(||Y_{l}-Y_{l-1}||^{2})/p_{l}\) is finite  then \(Z\) is square-integrable, \(E(Z)=\mu_{Y}\), and
\begin{equation*}
E(||Z||^{2})=\sum^{\infty}_{l=0}\frac{E(||Y_{l}-Y_{l-1}||^{2})}{p_{l}}.
\end{equation*}
\end{theorem}
In many applications, the expected time to simulate \(Y_{l}\) goes to infinity as \(l\) goes to infinity, but the \(p_{l}\)'s can be chosen so that the expected time to simulate \(Z\) is finite.   

Consider now a positive-definite \(d\times d\) matrix \(\Sigma\). Corollary~\ref{cor:Glynn} shows that the conclusions of  Theorem~\ref{th:Glynn} still hold if the Euclidean norm \(||.||\) is replaced with \(||.||_{\Sigma}\).     
\begin{corollary}\label{cor:Glynn}Assume that \(E(||Y_{l}||^{2})\) is finite for any \(l\geq0\). Set \(Z:=(Y_{N}-Y_{N-1})/p_{N}\), with  \(Y_{-1}:=0\). If \(\sum^{\infty}_{l=0}E(||Y_{l}-Y_{l-1}||_{\Sigma}^{2})/p_{l}\) is finite  then \(Z\) is square-integrable, \(E(Z)=\mu_{Y}\), and
\begin{equation*}
E(||Z||_{\Sigma}^{2})=\sum^{\infty}_{l=0}\frac{E(||Y_{l}-Y_{l-1}||_{\Sigma}^{2})}{p_{l}}.
\end{equation*}
\end{corollary}
Corollary~\ref{cor:Glynn} follows immediately by applying  Theorem~\ref{th:Glynn}  to the sequence \((\sqrt{\Sigma}Y_{l},l\geq0)\).

\subsection{Unbiased estimator construction  }\label{sub:unbiasedLongRun}
This subsection as well as Section \ref{sub:avgStart} assume that \(H\) is positive-definite. Thus \( \lambda_{\min}>0\) and there is a unique \(d\)-dimensional vector \(\theta^{*}\)  satisfying \eqref{eq:NormalEq}. For  \(k\geq2\), we set \(s(k)=\alpha k\), where  \(\alpha\in(0,1/2]\) is a constant,  ignoring integrality constraints.\(\) We build the unbiased estimator     \(\hat f_{k}\) of \(\theta^{*}\) along the following steps:\begin{enumerate}
\item 
We first construct a random family \((f_{k,l},l\geq0)\)  of  square-integrable \(d\)-dimensional column vectors such that \(E(f_{k,l})\rightarrow \theta^{*}\) as \(l\) goes to infinity.
\item Using Corollary~\ref{cor:Glynn}, we construct the unbiased estimator \(Z_{k}\) of  \(\theta^{*}-E(\bar\theta_{k})\). Combining \(Z_{k}\) and  \(\bar\theta_{k}\) yields  \(\hat f_{k}\).
 \end{enumerate}
 \subsubsection{Construction of \((f_{k,l},l\geq0)\)}Extend the random sequence \(((x_{t},y_{t}), t\geq0)\) to all \(t\in\mathbb{Z}\), so that \((x_{t},y_{t})\), \(t\in\mathbb{Z}\), are independent and identically distributed. For \(m\in\mathbb{Z}\) and \(t\geq m\), define \(\theta_{m:t}\)  by setting \(\theta_{m:m}=\theta_{0}\) and  
\begin{equation}\label{eq:recXim}
 \theta_{m:t+1}=\theta_{m:t}-\gamma ({ x_{t}}^T\theta_{m:t}-y_{t})x_{t},
\end{equation}
\(t\geq m\). Due to the analogy between \eqref{eq:recXim} and \eqref{eq:BasicSGD}, we have \(\theta_{0:t}=\theta_{t}\) for \(t\geq0\) and, for fixed \(m\), the sequence   \((\theta_{m:t}, t\geq m)\) is a Markov chain that is a copy of    \((\theta_{t}, t\geq 0)\).   In other words, the sequence    \((\theta_{m:t}, t\geq m)\) is a shifted version of    \((\theta_{t}, t\geq 0)\), for any fixed  \(m\in\mathbb{Z}\).  Consequently, we have \(\theta_{m:t}\sim \theta_{t-m}\) for \(m\in\mathbb{Z}\) and \(t\geq m\), where `\(\sim\)' denotes equality in distribution.     

For  \(k\ge2\) and \(l\ge0\), set \begin{equation}\label{eq:defFkl}
f_{k,l}:=\theta_{-\tau (k,l):0},
\end{equation} where \(\tau(k,0)\) is uniformly distributed on \(\{s(k),\dots,k-1\} \) and \(\tau (k,l)\) is uniformly distributed on \(\{k(2^{l}-1),\dots,k(2^{l+1}-1)-1\} \) for \(l\geq1\). We assume that the random variables \(\tau (k,l)\), \(l\geq0\), and  \((x_{t},y_{t})\), \(t\in\mathbb{Z}\), are independent. Note that \(\tau (k,l)<\tau (k,l+1)\)  for \(k\ge2\) and \(l\ge0\).

In view of the definition of \(f_{k,l}\), for  \(k\ge2\) and \(l\ge0\), we have \(f_{k,l}\sim\theta_{\tau (k,l)}\). Because \(\tau (k,l)\) ranges over a finite set  and \(\theta_{t}\) is square-integrable for  \(t\geq0\), a standard argument shows that \(f_{k,l}\) is also square-integrable for  \(k\ge2\) and \(l\ge0\). Moreover, \(E(f_{k,l})\) is equal to the average of \(E(\theta_{t})\), as \(t\) ranges over the possible values of \(\tau (k,l)\). By the definition of  \(\tau(k,0)\),  it follows that   \begin{equation}\label{eq:fklSimilarity}
E(f_{k,0})=E(\bar\theta_k),
\end{equation} for \(k\geq2\). Moreover, for fixed \(k\geq2\),  the inequality  \(\tau(k,l)\ge k(2^{l}-1)\) for \(l\geq0\) and  Lemma~\ref{le:BiasedlinReg} show that \(E(f_{k,l})\rightarrow \theta^{*}\) as \(l\) goes to infinity.  A similar argument shows that  \(E(f_{k,l})\rightarrow \theta^{*}\) as \(k\) goes to infinity, for any fixed \(l\geq0\). Furthermore, because   \(\tau(k,l)\ge s(k)\) for \(l\geq0\),    the last \(s(k)\) copies of \((x,y)\) used to simulate \(f_{k,0}\) and \(f_{k,1}\) are the same. Intuitively, this suggests that, under suitable mixing conditions, the vectors  \(f_{k,0}\) and \(f_{k,1}\) are ``close'' to each other if \(s(k)\) is large. In the same vein,  for \(l\geq1\),  the last \(k(2^{l}-1)\) copies of \((x,y)\) used to simulate \(f_{k,l}\) and \(f_{k,l+1}\) are the same. Hence, for large values of \(l\), the vector  \(f_{k,l+1}\) should be ``close'' to \(f_{k,l}\), even if \(k\) is small. This intuitive argument is formalized by Lemma~\ref{le:upperBoundDiffY} below. Set \begin{displaymath}
c:= \frac{\sigma^{2}d}{( 2 -\gamma R^{2})^{2}}+\frac{||\theta_{0}-\theta^{*}||^{2}}{\gamma(2 -\gamma R^{2})}.
\end{displaymath}  
\begin{lemma}\label{le:upperBoundDiffY}
For \(k\ge2\) and  \(l\geq1\),  we have \(E(||f_{k,1}-f_{k,0}||_{H}^{2})\leq4 c/s(k)\) and \begin{equation*}
E(||f_{k,l+1}-f_{k,l}||_{H}^{2})\leq\frac{6c}{k2^{l}}.
\end{equation*}
\end{lemma}
The bounds in Lemma~\ref{le:upperBoundDiffY} do not depend on the smallest eigenvalue \( \lambda_{\min}\) of \(H\), and the second one decreases proportionally  to \(k2^{l}\). Lemma~\ref{le:upperBoundDiffYExp} provides an alternative bound that decreases exponentially with \(k2^{l}\) but depends on \( \lambda_{\min}\). Let \begin{displaymath}
\xi:=\gamma( 2 -\gamma R^{2}) \lambda_{\min}.
\end{displaymath}        
  
\begin{lemma}\label{le:upperBoundDiffYExp}
We have \(\xi\leq1\) and, for  \(k\ge2\) and  \(l\geq0\), \begin{equation*}
E(||f_{k,l+1}-f_{k,l}||_{H}^{2})\leq6c\exp(-\xi s( k)2^{l}).
\end{equation*}
\end{lemma}
\subsubsection{Construction of \(Z_{k}\) and  of \(\hat f_{k}\)}\label{subsub:ZkConstruction}
Fix \(\delta>1\) and set \begin{displaymath}
p_{l}=\frac{1}{2^{l}(l+1)^{\delta}}-\frac{1}{2^{l+1}(l+2)^{\delta}},
\end{displaymath} for \(l\ge0\). Thus,  \((p_{l},l\geq0)\) is a probability distribution and, for \(l\ge0\), we have \begin{equation}\label{eq:plbounds}
1\leq2^{l+1}(l+1)^{\delta}p_{l}\le2.
\end{equation}  For \(k\ge2\), let\begin{equation*}
Z_{k}:=\frac{f_{k,N+1}-f_{k,N}}{p_{N}},
\end{equation*}
 where  \(N\in\mathbb{N}\)  is an integer-valued  random variable independent of  \(((x_{t},y_{t}), t\in\mathbb{Z})\)  such that \(\Pr(N=l)=p_{l}\)  for \(l\geq0\). Lemma~\ref{le:ZkProperties} provides  bounds   on   \(E(||Z_{k}^{2}||_{H})\) and shows that \(Z_{k}\) is  an unbiased estimator of the negated  bias \(\theta^{*}-E(\bar\theta_k)\) of \(\bar\theta_k\).    
 
\begin{lemma}\label{le:ZkProperties} For \(k\ge2\), the random vector \(Z_{k}\) is square-integrable, with  \(E(Z_{k})=\theta^{*}-E(\bar\theta_k)\) and\begin{equation}
\label{eq:ZNSecondMomentGeneralCase}
kE(||Z_{k}||^{2}_{H})\le \frac{11c}{\alpha}+6c\left(4\ln\left(\frac{4\delta+8}{\alpha\xi}\right)\right)^{\delta+1}.
\end{equation}
 Furthermore, if  \begin{equation}\label{eq:ConditionkforExpoDecay}
k\ge\frac{4\delta+8}{\alpha\xi}\ln\left(\frac{4\delta+8}{\alpha\xi}\right),
\end{equation} then\begin{equation}\label{eq:kZkexponentialDec}
kE(||Z_{k}||^{2}_{H})\leq \frac{24c}{k^{\delta+1}}\exp(-\alpha\xi   k/2).
\end{equation}      
\end{lemma}
Note that the bound \eqref{eq:kZkexponentialDec} is tighter than the bound \eqref{eq:ZNSecondMomentGeneralCase} for large values of \(k\). On the other hand, the lower bound on \(k\)  
required by \eqref{eq:ConditionkforExpoDecay} can be very large if \(\lambda_{\min}\) is small.

Fix \(q\in(0,1]\) and let  \(Q\) be a Bernoulli random variable independent of  \(\bar\theta_k\) with \(\Pr(Q=1)=q\). Set \(\hat f_{k}:=\bar\theta_k+q^{-1}QZ'_{k}\), where \(Z'_{k}\) is a copy of \(Z_{k}\) independent of \((\bar\theta_{k},Q)\). Throughout
the rest of the paper, the running time is measured by the number of simulated copies of \(x\). Let \(T_{k}\) (resp. \(\hat T_{k}\)) be the expected time required to sample \(Z_{k}\) (resp. \(\hat f_{k}\)). Theorem \ref{th:BoundhHatfk} gives a  bound  on \(T_{k}\)  and on  \(\hat T_{k}\)
and shows that  \(\hat f_{k}\) is an unbiased estimator of \(\theta^{*}\). Moreover, it provides a bound on the expected excess risk \(E(L(\hat f_{k}))-L(\theta^{*})\).       
\begin{theorem}\label{th:BoundhHatfk}
Assume that \((x,y)\) is square-integrable, that  \(((x_{t},y_{t}), t\in\mathbb{Z})\) is a sequence of independent copies of  \((x,y)\), and that Assumption 1 holds.  Assume further  that \(0<\gamma<2/R^{2}\) and \(0<\alpha\le1/2\), that  \(q\in(0,1]\), and that \(H\) is positive-definite. Then, for \(k\geq2\), we have \(T_{k}\leq4\delta k/(\delta-1)\) and  \(\hat T_{k}\leq k +qT_{k}\).  Furthermore, we have    \(E(\hat f_{k})=\theta^{*}\) and\begin{equation}\label{eq:optGapRMLMC}
E(L(\hat f_{k}))-L(\theta^{*}) \le\frac{1}{2qk}\left(\frac{27c}{\alpha}+6c\left(4\ln\left(\frac{4\delta+8}{\alpha\xi}\right)\right)^{\delta+1}\right).
\end{equation}
\end{theorem} 
For fixed values of \(\alpha\) and \(q\) satisfying the conditions of Theorem \ref{th:BoundhHatfk}, and for \(0<\gamma\le 1/R^{2}\), the bound on the expected excess risk of  \(\hat f_{k}\) provided by \eqref{eq:optGapRMLMC}  is equal to the bound on the expected excess risk of   \(\bar\theta_{k}\) given by \eqref{eq:thMainPositiveSemiDefinite}, up to a poly-logarithmic factor. \subsection{Unbiased estimator with average start}\label{sub:avgStart}
This subsection describes an average start version of \(Z_{k}\) and of \(\hat f_{k}\) obtained by replacing \(\theta_{-\tau (k,l):0}\) in \eqref{eq:defFkl} with the average of \(\theta_{-t}\), where \(t\) ranges over the possible values  of \(\tau (k,l)\). More precisely, for  \(k\ge2\), set \begin{equation}\label{eq:defFklAvg0}
f_{\text{avg},k,0}:=\frac{1}{k-s(k)}\sum^{k-1}_{t=s(k)}\theta_{-t:0},
\end{equation} and, for \(l\ge1\), set\begin{equation}\label{eq:defFklAvgl}
f_{\text{avg},k,l}:=\frac{1}{k2^{l}}\sum^{k(2^{l+1}-1)-1}_{t=k(2^{l}-1)}\theta_{-t:0}.
\end{equation} 
For \(k\ge2\), set\begin{equation*}
Z_{\text{avg},k}:=\frac{f_{\text{avg},k,N+1}-f_{\text{avg},k,N}}{p_{N}},
\end{equation*}
where \(N\) is defined as in Section \ref{subsub:ZkConstruction}. Set \(\hat f_{\text{avg},k}:=\bar\theta_k+q^{-1}QZ'_{\text{avg},k}\), where \(q\) and  \(Q\) are defined as in Section \ref{subsub:ZkConstruction} and \(Z'_{\text{avg},k}\) is a copy of \(Z_{\text{avg},k}\) independent of \((\bar\theta_{k},Q)\). 

For \(k\geq2\) and \(l\geq0\) and any \(j\geq k(2^{l+1}-1)\), we show that \begin{equation}\label{eq:conditionalExpAveragefkl}
f_{\text{avg},k,l}=E(f_{k,l}|(x_{t},y_{t}),-j\le t\le0).
\end{equation}To prove  \eqref{eq:conditionalExpAveragefkl}, consider first the case where \(l=0\). We  have  \begin{eqnarray*}
E(f_{k,0}|(x_{t},y_{t}),-j\le t\le0)&=&E(\theta_{-\tau (k,0):0}|(x_{t},y_{t}),-j\le t\le0)\\
&=&E\left(\frac{1}{k-s(k)}\sum^{k-1}_{i=s(k)}\theta_{-i:0}|(x_{t},y_{t}),-j\le t\le0\right)\\
&=&E( f_{\text{avg},k,0}|(x_{t},y_{t}),-j\le t\le0)\\
&=&f_{\text{avg},k,0}.
\end{eqnarray*} 
The first and third equations follow from the definitions of \(f_{k,0}\) and \( f_{\text{avg},k,0}\).  The  second equation follows from the independence of  \(\tau (k,0)\) and   \((x_{t},y_{t})\), \(t\in\mathbb{Z}\).   The last equation holds because, by \eqref{eq:recXim} and \eqref{eq:defFklAvgl}, \(f_{\text{avg},k,l}\) is a deterministic function of  \((x_{t},y_{t})\), \(-j\le t\le0\). Thus  \eqref{eq:conditionalExpAveragefkl}  holds if \(l=0\), and the case where \(l\ge1\) can be handled in a similar way. 

Taking expectations in \eqref{eq:conditionalExpAveragefkl} and using the tower law implies that \(E(f_{\text{avg},k,l})=E(f_{k,l})\) for    \(k\geq2\) and \(l\geq0\).  Moreover, it follows from \eqref{eq:conditionalExpAveragefkl} that, for \(k\geq2\) and \(l\geq0\), we have \begin{equation*}
f_{\text{avg},k,l+1}-f_{\text{avg},k,l}=E(f_{k,l+1}-f_{k,l}|(x_{t},y_{t}),-k2^{l+2}\le t\le0).
\end{equation*} Since \(E(||E(U|F)||_{H}^{2})\le E(||U||_{H}^{2})\) for any square-integrable \(d\)-dimensional column vector \(U\) and any random variable \(F\), it follows that, for \(k\geq2\) and \(l\geq0\),\begin{equation*}
E(||f_{\text{avg},k,l+1}-f_{\text{avg},k,l}||^{2}_{H})\leq E(||f_{k,l+1}-f_{k,l}||^{2}_{H}).
\end{equation*}  Hence  Lemmas~\ref{le:upperBoundDiffY} and Lemma~\ref{le:upperBoundDiffYExp}    remain valid if the \(f_{k,l}\)'s are replaced with \(f_{\text{avg},k,l}\). The proof of Lemma \eqref{le:ZkProperties} also shows that it still holds if \(Z_{k}\) is replaced with \(Z_{\text{avg},k}\). 

Note that \(f_{\text{avg},k,0}\) can be simulated via \eqref{eq:defFklAvg0} and \eqref{eq:recXim} by sampling \((x_{t},y_{t})\), where \(t\) ranges from \(1-k\) to \(-1\), that is, by simulating \(k-1\) copies of \((x,y)\). However, as the number of vector operations needed to calculate \(\theta_{-t:0}\) is of order \(t\), the total number of vector operations required to calculate  \(f_{k,0}\)  via \eqref{eq:defFklAvg0} and \eqref{eq:recXim} is of order \(k^{2}\). Similarly, for \(l\geq1\),   \(f_{\text{avg},k,l}\) can be calculated via \eqref{eq:defFklAvgl} and \eqref{eq:recXim} by sampling \(O(k2^{l})\) copies of \((x,y)\) with a total number of  \(O(k^{2}2^{2l})\) vector operations. Proposition \ref{pr:efficientRunningTime} shows that, for \(l\geq0\),   \(f_{\text{avg},k,l}\) can be simulated   more efficiently by sampling   \(O(k2^{l})\)  copies of \((x,y)\)   with a total number of  \(O(k2^{l})\) vector operations. More generally, given \(m\le m'\leq0\), Proposition \ref{pr:efficientRunningTime} shows how to simulate  \({(m'+1-m)^{-1}}\sum^{m'}_{h=m}\theta_{h:0}\) with \(O(|m|)\) vector operations and sampling  \((x_{t},y_{t})\), \(m\leq t\leq -1\), i.e., by simulating \(|m|\) copies of \((x,y)\). The proof of Proposition \ref{pr:efficientRunningTime} relies on the affine dependence of  \(\theta_{m:t+1}\) on   \(\theta_{m:t}\) in  \eqref{eq:recXim}. 
\begin{proposition}\label{pr:efficientRunningTime}
Given \(m\le m'\leq0\), define recursively the sequence \((\eta_{t}, m\leq t \leq0)\) by setting \(\eta_{m}=\theta_{0}\) and \begin{equation}\label{eq:etaDef}\eta_{t+1}=\begin{cases}\frac{1}{t+2-m}\biggl((t+1-m)(\eta_{t}-\gamma ({ x_{t}}^T\eta_{t}-y_{t})x_{t})+\theta_{0}\biggr)  & \text{if }m\leq t\leq m'-1\ \\
\eta_{t}-\gamma ({ x_{t}}^T\eta_{t}-y_{t})x_{t} & \text{if } m'\le t\le-1.
\end{cases}\end{equation} 
Then \(\eta_{0}={(m'+1-m)^{-1}}\sum^{m'}_{h=m}\theta_{h:0}\).
\end{proposition}
Let \(T_{\text{avg},k}\) (resp. \(\hat T_{\text{avg},k}\)) be the expected time required to simulate \(Z_{\text{avg},k}\) (resp. \(\hat f_{\text{avg},k}\)) by applying Proposition \ref{pr:efficientRunningTime}. Theorem \ref{th:BoundhHatfkAvg} gives a  bound  on \(T_{\text{avg},k}\)  and on  \(\hat T_{\text{avg},k}\)
and  provides a bound on the expected excess risk of \(\hat f_{\text{avg},k}\). These bounds are the same as their counterparts in Theorem \ref{th:BoundhHatfk}.  The proof of Theorem \ref{th:BoundhHatfkAvg} is analogous to that of Theorem \ref{th:BoundhHatfk} and is omitted. \begin{theorem}\label{th:BoundhHatfkAvg}
Assume that the conditions of Theorem \ref{th:BoundhHatfk} hold. Then, for \(k\geq2\), we have \(T_{\text{avg},k}\leq4\delta k/(\delta-1)\) and  \(\hat T_{\text{avg},k}\leq k +qT_{\text{avg},k}\).  Moreover, we have    \(E(\hat f_{\text{avg},k})=\theta^{*}\) and\begin{equation*}
E(L(\hat f_{\text{avg},k}))-L(\theta^{*}) \le\frac{1}{2qk}\left(\frac{27c}{\alpha}+6c\left(4\ln\left(\frac{4\delta+8}{\alpha\xi}\right)\right)^{\delta+1}\right).
\end{equation*}
\end{theorem} 

\section{Squared Bias and Variance Estimation}\label{se:BiasVarEst}
Lemma \ref{le:BiasedlinReg} provides analytical upper bounds on the squared bias, variance and expected excess risk of \(\bar\theta_{k}\), while Theorem \ref{th:BoundhHatfk} (resp. Theorem \ref{th:BoundhHatfkAvg}) gives an analytical upper bound on the expected excess risk of \(\hat f_{k}\) (resp. \(\hat f_{\text{avg},k}\)). This section constructs unbiased estimators  of these metrics. Section \ref{sub:sqBiasEst} gives both biased and unbiased estimators of the squared bias of \(\bar\theta_{k}\), while Section \ref{sub:VarianceEst} provides unbiased estimators of  \(\var_{H}(\bar\theta_{k})\) and    \(\var_{H}(\hat f_{k})\).    Together with the results in Section \ref{sub:BiasVariance}, this yields  unbiased estimators of the expected excess risk of \(M\) independent copies of  \(\bar\theta_{k}\),  \(\hat f_{k}\) and \(\hat f_{\text{avg},k}\), for any integer \(M>1\).   All results in this section relating to \(Z_{k}\) (resp.  \(\hat f_{k}\))  also apply to  \(Z_{\text{avg},k}\) (resp.  \(\hat f_{\text{avg},k})\) without any changes.

\subsection{Estimating the Squared Bias}\label{sub:sqBiasEst}
Proposition \ref{pr:EstE|| ||_H} shows how to estimate \(E(||U||^{2}_{H})\), where \(U\) is a square-integrable \(d\)-dimensional column vector, without needing to know \(H\). \begin{proposition}\label{pr:EstE|| ||_H}For any square-integrable \(d\)-dimensional column vector \(U\) independent of \(x\), we have \(E\left(\left(U^Tx\right)^{2}\right)=E(||U||^{2}_{H})\). \end{proposition}
\begin{proof}We have\begin{eqnarray*}E\left(\left(U^Tx^{}\right)^{2}\right)&=&E\left(U^{T}x{x}^TU\right)\nonumber\\
&=&E\left(U^{T}E\left(xx^T\right)U\right)\nonumber\\
&=&E(U^{T}HU)\nonumber\\
&=&E(||U||^{2}_{H}),\end{eqnarray*}
where the second equation follows from the independence of \(U\) and \(x\), and the third one from the definition of \(H\). 
\end{proof}
Proposition \ref{pr:SquaredBiasBiasedEstimator} gives an estimator of the squared bias of \(\bar\theta_{k}\) without using either \(H\) or \(\theta^{*}\).
 \begin{proposition}\label{pr:SquaredBiasBiasedEstimator}
Assume that the conditions of Theorem \ref{th:BoundhHatfk} hold. Fix \(k\geq2\),  \(n\geq1\) and \(n'\geq1\). Let  \(Z_{k}^{(i)}\), \(1\le i\leq n\), be \(n\) copies of \(Z_{k}\), and let  \(x^{(i)}\),  \(1\le i\leq n'\), be \(n'\) copies of \(x\) such that   \(Z_{k}^{(1)},\dots,Z_{k}^{(n)},x^{(1)},\dots,x^{(n')}\) are  independent.  Set \(\bar Z_{k,n}=n^{-1}\sum^{n}_{i=1}Z_{k}^{(i)}\). Then\begin{equation}\label{eq:BarZkn}
E\left(\frac{1}{n'}\sum^{n'}_{j=1}(\bar Z_{k,n}^{T}x^{(j)})^{2}\right)=||E(\bar\theta_{k})-\theta^{*}||^{2}_{H}+\frac{\var_{H}( Z_{k})}{n}.
\end{equation}   \end{proposition}
\begin{proof}
 
 For \(j\geq1\), since \(\bar Z_{k,n}\) and \(x^{(j)}\) are independent, it follows from Proposition \ref{pr:EstE|| ||_H} that\begin{displaymath}
E(\bar Z_{k,n}^{T}x^{(j)})^{2}=E(||\bar Z_{k,n}||^{2}_{H}).\end{displaymath}On the other hand,\begin{eqnarray*}E(||\bar Z_{k,n}||^{2}_{H})
&=&||E(\bar Z_{k,n})||^{2}_{H}+\var_{H}(\bar Z_{k,n})\\
&=&||E(Z_{k})||^{2}_{H}+\frac{\var_{H}(Z_{k})}{n}\\&=&||E(\bar\theta_{k})-\theta^{*}||^{2}_{H}+\frac{\var_{H}( Z_{k})}{n},
\end{eqnarray*}
where the first equation follows from the definition of \(\var_{H}\), the second one by noting that  \(\bar Z_{k,n}\) is the average of \(n\) independent copies of  \(Z_{k}\), and the last one from Lemma~\ref{le:ZkProperties}. Combining the two above equations and summing over \(j\) implies \eqref{eq:BarZkn}.  
\end{proof}
Note that the the right-hand side of  \eqref{eq:BarZkn} converges to  the squared bias of \(\bar\theta_{k}\) as \(n\) goes to infinity. On the other hand, by the law of large numbers,  \(\bar Z_{k,n}\) converges almost surely to \(E(Z_{k})\) as  \(n\) goes to infinity. Thus the summands \((\bar Z_{k,n}^{T}x^{(j)})^{2}\), \(1\leq j\leq n'\), are ``almost'' independent when \(n\) is large. Using the law of large numbers once again, the average  \(n'^{-1}\sum^{n'}_{j=1}(\bar Z_{k,n}^{T}x^{(j)})^{2}\) should be ``close'' to its expectation given by the right-hand side of  \eqref{eq:BarZkn}, when \(n\) and \(n'\) are large. This intuitive argument suggests that \(n'^{-1}\sum^{n'}_{j=1}(\bar Z_{k,n}^{T}x^{(j)})^{2}\) is an accurate estimator of  the squared bias of \(\bar\theta_{k}\) if \(n\) and \(n'\) are large.
    
As the expected time to simulate
\(\bar Z_{k,n}\) is of order \(kn\), we recommend to choose \(n'\) to be of order \(kn\) in practical implementations, so that the total expected time to sample the above estimator is of order  \(kn\) as well.
It follows from  \eqref{eq:BarZkn} that this estimator is in general biased, however. 

Proposition \ref{pr:SquaredBiasUnbiasedEstimator} provides an unbiased estimator of the squared bias of \(\bar\theta_{k}\) by sampling \(Z_{k}\) twice and sampling copies of \(x\), without requiring knowledge of \(H\) or of \(\theta^{*}\).   
\begin{proposition}\label{pr:SquaredBiasUnbiasedEstimator}
Assume that the conditions of Theorem \ref{th:BoundhHatfk} hold. Let \(k\geq2\) and let  \(\tilde Z_{k}\) be a copy of \(Z_{k}\) and let \(x^{(1)},\dots,x^{(n')}\) be \(n'\) copies of \(x \) such that     \(Z_{k}\),  \(\tilde Z_{k}\), \(x^{(1)},\dots,x^{(n')}\) are independent, where \(n'\) is a positive integer.  Then
\begin{equation}\label{eq:avegEstimatorBiased}
E\left(\frac{1}{n'}\sum^{n'}_{j=1}(Z_{k}^{T}x^{(j)} )(\tilde Z_{k}^T x^{(j)})\right)=||E(\bar\theta_{k})-\theta^{*}||^{2}_{H}.
\end{equation}\end{proposition}
\begin{proof}
For \(1\leq j\leq n'\), we have
\begin{eqnarray*}E\left((Z_{k}^{T}x^{(j)})(\tilde Z_{k}^T x^{(j)})\right)&=&E\left(Z_{k}^T x^{(j)}{x^{(j)}}^T \tilde Z_{k}\right)\\&=&
E\left(Z_{k}\right)^T  E\left(x^{(j)}{x^{(j)}}^T\right) E\left( \tilde Z_{k}\right)\\
&=&(\theta^{*}-E(\bar\theta_k))^{T}H(\theta^{*}-E(\bar\theta_k))\\
&=&||E(\bar\theta_{k})-\theta^{*}||^{2}_{H},
\end{eqnarray*}
where the second equation follows from the independence of \(Z_{k}\),  \(\tilde Z_{k}\)  and  \(x^{(j)},\)  and the third one from Lemma \ref{le:ZkProperties} and the definition of  \(H\).
Summing over \(j\) achieves the proof.\end{proof}

It follows from   \eqref{eq:avegEstimatorBiased} that \(n'^{-1}\sum^{n'}_{j=1}(Z_{k}^{T}x^{(j)} )(\tilde Z_{k}^T x^{(j)})\) is an unbiased  estimator  of \(||E(\bar\theta_{k})-\theta^{*}||^{2}_{H}\)    for any integer \(n'\ge1\). As the expected running time of \(Z_{k}\) is of order \(k\),  we suggest to choose a value of \(n'\) of order \(k\)  in practical implementations, so that the  expected   time to sample this estimator is of order \(k\) as well. We can then estimate \(||E(\bar\theta_{k})-\theta^{*}||^{2}_{H}\)  by taking the average of several  independent copies of the resulting estimator.

\subsection{Estimating the Variance}\label{sub:VarianceEst}
Proposition \ref{pr:varianceEstimation} constructs two unbiased estimators of \(\var_{H}(\psi)\), where  \(\psi\)  is a square-integrable \(d\)-dimensional  column vector. The first one  uses \(H\)  and simulates only copies   of \(\psi\). The second one does not require  knowledge of \(H\)  and simulates copies of \(x\) and  of \(\psi\). 
\begin{proposition}\label{pr:varianceEstimation}
Let \(\psi\) be a square-integrable  \(d\)-dimensional column vector. Fix   integers  \(n\ge2\) and \(n',K\ge1\). Let \((\psi^{(i)}, 1\le i\le n)\) be copies of \(\psi\)  and  \((x^{(i,j)}, 1\le i\le n, 1\le j\le K)\) and  \((x^{(j)},1\leq j\leq n')\) be copies of \(x \). Assume   that these random  variables are mutually independent. Set \(\bar\psi=\bar\psi_{n}\), where \(\bar\psi_n\) is given by \eqref{eq:PsiBarDef}, with \(M\) being replaced by \(n\). 
 Then\begin{equation}\label{eq:EmpVarPsiFirst}
E\left(\sum^{n}_{i=1}||\psi^{(i)}||_{H}^{2}-n||\bar\psi||^{2}_{H}\right)=(n-1)\var_{H}(\psi),
\end{equation}
and \begin{equation}\label{eq:EmpVarPsiEst}
E\left(\frac{1}{K}\sum^{n}_{i=1}\sum^{K}_{j=1}\left({\psi^{(i)}}^Tx^{(i,j)}\right)^{2}-\frac{n}{n'}\sum^{n'}_{j=1}\left({\bar\psi}^Tx^{(j)}\right)^{2}\right)=(n-1)\var_{H}(\psi).
\end{equation}
\end{proposition}
\begin{proof}
 Since the \(\psi^{(i)}\)'s are copies of \(\psi\), we have
\begin{equation*}
E(\sum^{n}_{i=1}||\psi^{(i)}||_{H}^{2}-n||\bar\psi||^{2}_{H})=nE(||\psi||_{H}^{2})-nE(||\bar\psi||^{2}_{H}).
\end{equation*}
Noting that \(E(||\psi||_{H}^{2})=||E(\psi)||_{H}^{2}+\var_{H}(\psi)\) and \begin{eqnarray*}E(||\bar\psi||_{H}^{2})&=&||E(\bar\psi)||_{H}^{2}+\var_{H}(\bar\psi)\\
&=&||E(\psi)||_{H}^{2}+\frac{\var_{H}(\psi)}{n},\end{eqnarray*}
this implies \eqref{eq:EmpVarPsiFirst}. 
On the other hand, Proposition \ref{pr:EstE|| ||_H} shows that, for \(1\leq i\leq n\) and \(1\leq j\leq K\),\begin{equation*}
E(||\psi^{(i)}||_{H}^{2})=E\left(\left({\psi^{(i)}}^Tx^{(i,j)}\right)^{2}\right),
\end{equation*}
and, for \(1\leq j\leq n'\), \begin{equation*}
E(||\bar\psi||_{H}^{2})=E\left(\left({\bar\psi}^Tx^{(j)}\right)^{2}\right).
\end{equation*}
Summing over \(i\) and \(j\) and using \eqref{eq:EmpVarPsiFirst} implies \eqref{eq:EmpVarPsiEst}. 
\end{proof}

For general square-integrable \(\psi\), the left-hand sides of \eqref{eq:EmpVarPsiFirst} and of \eqref{eq:EmpVarPsiEst} provide unbiased estimators of  \(\var_{H}(\psi)\).
For the estimator  resulting from \eqref{eq:EmpVarPsiEst}   and a given \(n\ge2\), by analogy to Propositions \ref{pr:SquaredBiasBiasedEstimator} and \ref{pr:SquaredBiasUnbiasedEstimator}, we suggest to select a value of \(K\) of the same order of magnitude as the expected time to simulate \(\psi^{(1)}\), and a value of \(n'\) of the same order of magnitude as the total expected time to sample \(\psi^{(1)},\dots,\psi^{(n)}\). When \(\psi = \bar\theta_{k}\) (resp.  \(\psi = \hat f_{k}\)), the left-hand sides of \eqref{eq:EmpVarPsiFirst} and of \eqref{eq:EmpVarPsiEst} provide unbiased estimators of \(\var_{H}(\bar\theta_{k})\) (resp. \(\var_{H}(\hat f_{k})\)), for any integers \(n\ge 2\) and  \(n',K\ge1\).

In summary,  Proposition \ref{pr:SquaredBiasUnbiasedEstimator} provides an unbiased estimator of  \(||E(\bar\theta_{k})-\theta^{*}||^{2}_{H}\), while Proposition \ref{pr:varianceEstimation}  provides an unbiased estimator of  \(\var_{H}(\bar\theta_{k})\) and of \(\var_{H}(\hat f_{k})\). Combining these estimators and using \eqref{eq:BiasVarianceOptGap}  
 yields an unbiased estimator of \(E(L(\bar\theta_{k})))-L( \theta ^{*})\) and of \(E(L(\hat f_{k})))-L( \theta ^{*})\). More generally, combining these estimators and using \eqref{eq:ExpectedOptGapM} yields an unbiased estimator of  \(E(L(\bar\psi_{M})))-L( \theta ^{*})\), for  \(\psi = \bar\theta_{k}\) or  \(\psi = \hat f_{k}\), and any positive integer \(M\). We can also use Proposition \ref{pr:SquaredBiasBiasedEstimator} instead of  Proposition \ref{pr:SquaredBiasUnbiasedEstimator} to estimate  \(||E(\bar\theta_{k})-\theta^{*}||^{2}_{H}\), but  the resulting estimator is biased. 
\subsection{Computational Cost Analysis}\label{sub:CCA}
   An immediate consequence of \eqref{eq:excessRiskAverageUnbiased} and of Theorem~\ref{th:BoundhHatfk} is the following.         
\begin{proposition}\label{pr:avghatf}Consider integers   \(k\geq2\) and   \(M\ge2\). Set   $\bar f_{k,M}=M^{-1}(\sum^{M}_{i=1}\hat f _{k}^{(i)})$,  where \(\hat f _{k}^{(1)},\dots, \hat f _{k}^{(M)}\) are \(M\) independent copies of \(\hat f _{k}\).   Under the assumptions of Theorem~\ref{th:BoundhHatfk},  we have

\begin{equation}\label{eq:avgOptimalityGap}
kM(E(L(\bar f_{k,M}))-L(\theta^{*}))\leq\frac{1}{2q}\left(\frac{27c}{\alpha}+6c\left(4\ln\left(\frac{4\delta+8}{\alpha\xi}\right)\right)^{\delta+1}\right).
\end{equation}
\end{proposition}
Fix \(\alpha\in(0,1/2]\) and  \(q\in(0,1]\). Observe that     \eqref{eq:avgOptimalityGap} implies an upper-bound on  \(E(L(\bar f_{k,M}))-L(\theta^{*})\) that depends on \(k\) and \(M\) only through the product \(kM\) which, by Theorem \ref{th:BoundhHatfk}, is proportional to the expected time required to calculate \(\bar f_{k,M}\). However, if \(kM=k'M'\) with \(k\geq k'\), we intuitively expect   \(E(L(\bar f_{k,M}))-L(\theta^{*})\) to be smaller than  \(E(L(\bar f_{k',M'}))-L(\theta^{*})\). This is because \(E(L(\bar\theta_{k}))-L(\theta^{*})\) ought to be smaller than     \(E(L(\bar\theta_{k'}))-L(\theta^{*})\), making it easier to correct the bias of \(\bar\theta_{k}\)  than that of \(\bar\theta_{k'}\).
On the other hand,
the calculation of \(\bar f_{k',M'}\) is easier to parallelize than that of \(\bar f_{k,M}\)  because it can be conducted on  \(M'\ge M\) independent processors. Selecting an appropriate value for  \(M\) depends on the target accuracy and on the number of processors used, and is a question beyond the scope of this paper.

Assume now for simplicity that \(\delta=2\) and that \(\alpha\) and \(q\) are fixed with \(\alpha,q\ge1/10\), and that \(\gamma=1/R^{2}\). Let \(\epsilon>0\). By \eqref{eq:GammaR21} and the definition of \(c\),   the  integer \(k\) needed to guarantee that \(E(L(\bar \theta _{k}))-L(\theta^{*})\leq\epsilon\) is \(k=k_{\epsilon}=O(c/\epsilon)\), where the constant behind the \(O\) notation is an absolute constant. On the other hand,   Theorem \ref{th:BoundhHatfk} shows that the expected time to simulate \(\bar f_{k,M}\) is at most \(10kM\). By    \eqref{eq:avgOptimalityGap}, it follows after some calculations that the expected time \(T_{\epsilon}\) needed to guarantee that \(E(L(\bar f_{k,M}))-L(\theta^{*})\leq\epsilon\) is \(T_{\epsilon}=O((\ln(2R^{2}/\lambda_{\min}))^{3}c/\epsilon)\), where the constant behind the \(O\) notation is an absolute constant. Thus, \(T_{\epsilon}\) is equal to \(k_{\epsilon}\), up to a poly-logarithmic factor. Note that the upper bound on  \(T_{\epsilon}\)  does not depend on \(k\).

\section{Numerical Experiments}\label{se:numer}
The numerical experiments were conducted on a laptop PC with an Intel  1 gigahertz processor and 8 gigabytes of RAM. The codes were written in the C++ programming language.  Two synthetic distributions for \((x,y)\) were used. Both distributions assume that \(x\) is a \(d\)-dimensional Gaussian vector, that  \(y=x^T\theta^{*}+V\), where \(V\) is a centered Gaussian random variable independent of \(x\) with variance \(\sigma^{2}\), and that all \(d\) components of   \(\theta^{*}\) are equal to             \(2 / \sqrt{d}\).  For both distributions,  the \(i\)-th component of  \(\theta_0\) is equal to            \(\cos(2 \pi i  / d) / \sqrt{d}\),  \(1\le i\le d\). Thus the order of magnitude of   \(||\theta_0||\), of     \(||\theta^{*}||\) and of    \(||\theta_0-\theta^{*}||\)   is   \(1\).  The first distribution is taken from \citet{BachLeastSquaresJMLR2017} and assumes that \(d=25\), with \(H=\hbox{\rm diag}(i^{-3})\), \(1\le i\le 25\), and that  \(\sigma^{2}=1\). The second distribution is taken from  
 \citet{jain2018parallelizing} and supposes that  \(H=\hbox{\rm diag}(i^{-1})\), \(1\le i\le 50\), and that \(\sigma^{2}=0.01\). Tables \ref{tab:BiasBach} and \ref{tab:BiasVarBach} and Fig. \ref{fig:TimeVarianceBach} and Fig. \ref{fig:ExcessRiskBach} are built from the first distribution, while Tables \ref{tab:BiasSidford} and \ref{tab:BiasVarSidford} and Fig. \ref{fig:TimeVarianceSidford} and Fig. \ref{fig:ExcessRiskSidford} are constructed from the second distribution.  A standard calculation shows that, for the two distributions, Assumption 1 holds with \(R^{2}=\tr(H)+2\lambda_{\max}\). In all  experiments, we set  \(\gamma=1/R^{2}\), with \(\alpha=1/2\) and \(\delta=3/2\). 

Table \ref{tab:BiasBach}  calculates the squared bias   \(||E(\bar\theta_{k})- \theta ^{*}||^{2}_{H}\), for several values of \(k\), with three  different methods. The first one assumes that \(H\) and \( \theta ^{*}\) are known and calculates exactly the squared bias. It first computes  \(E(\bar\theta_{k})\) by applying \eqref{eq:Ethetak} in the appendix to all \(t\in[s(k),k-1]\), then uses   \(E(\bar\theta_{k})\),  \( \theta ^{*}\)  and   \(H\) to calculate the squared bias. The remaining two methods do not require  knowledge of either \(H\) or   \( \theta ^{*}\). The second method (RMLMCB) uses the estimator of  Proposition \ref{pr:SquaredBiasBiasedEstimator}, where    \(Z_{k}\) was replaced by   \(Z_{\text{avg},k}\), with \(n=10^{8}/k\) and \(n'=10^{8}\). The third method (URMLMCB) takes the average of  \(n=10^{8}/k\)  independent copies of the estimator in Proposition \ref{pr:SquaredBiasUnbiasedEstimator}, where each estimator is calculated with \(n'=k\), with    \(Z_{k}\) being replaced by   \(Z_{\text{avg},k}\) once again. In the last two methods, the choice of \(n\) and \(n'\)  is motivated by practical considerations related to computing time. Indeed, by the discussions following  Propositions \ref{pr:SquaredBiasBiasedEstimator} and  \ref{pr:SquaredBiasUnbiasedEstimator}, this choice    ensures that  the order of magnitude of the running time of the corresponding estimators  is about \(10^{8}\),  for any \(k\). The variable ``Cost'' refers to the total number of simulated copies of \(x\), while the variable ``Time'' refers to the total running time in seconds. 

For the two methods RMLMCB and URMLMCB and all values of \(k\) in Table \ref{tab:BiasBach}, the estimated squared bias is within \(10\%\) from its exact value, and the variable  ``Cost'' is of order \(10^{8}\) or \(10^{9}\), which is consistent with our discussions above. Furthermore, the variable ``Time'' is roughly proportional to variable  ``Cost''. The running of the URMLMCB method is larger than that of RMLMCB because it simulates twice as many copies of \(Z_{k}\).
\begin{table}
\caption{Squared Bias of averaged SGD with \(H=\hbox{\rm diag}(i^{-3})\), \(1\le i\le 25\), \(\sigma^{2}=1\), \(s(k)=k/2\) and \(n=10^{8}/k\).}
\begin{tabular}{rrcccc}\hline
$k$  & Method  & $n$  & Sq. Bias  & Cost  &Time\\\hline
$50$ &Exact &  -- & $6.382\times10^{-3} $ &-- &--\\
$50$ &RMLMCB & $ 2\times10^6$ & $6.253\times10^{-3} $ &$6.28\times10^8 $ &$5.44\times10^2$\\
$50$ &URMLMCB & $ 2\times10^6$ & $5.892\times10^{-3} $ &$1.13\times10^9 $ &$9.56\times10^2$\\
$200$ &Exact &  -- & $3.729\times10^{-3} $ &-- &--\\
$200$ &RMLMCB & $ 5\times10^5$ & $3.657\times10^{-3} $ &$6.23\times10^8 $ &$5.24\times10^2$\\
$200$ &URMLMCB & $ 5\times10^5$ & $3.765\times10^{-3} $ &$1.19\times10^9 $ &$1.01\times10^3$\\
$800$ &Exact &   --  & $1.926\times10^{-3} $ &-- &--\\
$800$ &RMLMCB & $ 1.25\times10^5$ & $1.906\times10^{-3} $ &$6.29\times10^8 $ &$5.28\times10^2$\\
$800$ &URMLMCB & $ 1.25\times10^5$ & $1.906\times10^{-3} $ &$1.16\times10^9 $ &$9.78\times10^2$\\
$3200$ &Exact &   --  & $7.057\times10^{-4} $ &-- &--\\
$3200$ &RMLMCB & $ 3.125\times10^4$ & $7.123\times10^{-4} $ &$5.81\times10^8 $ &$4.89\times10^2$\\
$3200$ &URMLMCB & $ 3.125\times10^4$ & $7.093\times10^{-4} $ &$1.15\times10^9 $ &$1.01\times10^3$\\
\hline\end{tabular}
\label{tab:BiasBach}
\end{table}

Table \ref{tab:BiasVarBach}  simulates  \(n=10^{8}/k\)  independent copies  of \(\bar\theta_{k}\) (AvSGD), of \(\hat f_{k}\) (USGD) and of  \(\hat f_{\text{avg},k}\) (AUSGD) for several values of \(k\). For each of the estimators \(\hat f_{k}\) and \(\hat f_{\text{avg},k}\), the probability \(q\) is chosen so that the expected running time is twice that of \(\bar\theta_{k}\). 
Here again, the choice of \(n\) is prompted by practical considerations to    guarantee that  the running time of the different estimators  is of order \(10^{8}\) for all values of \(k\). The variable ``Sq. Dist.'' in Table \ref{tab:BiasVarBach}  computes the squared distance from \(\theta^{*}\), in the \(||.||_{H}\) norm, of the average of the \(n\) copies. In other words,  it calculates \(||\bar\psi-\theta^{*}||^{2}_{H}\), where \(\psi\) is equal to \(\bar\theta_{k}\), \(\hat f_{k}\), or  \(\hat f_{\text{avg},k}\), and \(\bar\psi\) is the average of  \(n\) independent copies \(\psi^{(1)},\dots,\psi^{(n)}\)  of \(\psi\). Based on  \eqref{eq:EmpVarPsiFirst}, the  variable ``Variance'' estimates  \(\var_{H}(\psi)\) via the formula \((n-1)^{-1}(\sum^{n}_{i=1}||\psi^{(i)}||_{H}^{2}-n||\bar\psi||^{2}_{H})\), where \(\psi\) is one of the three aforementioned estimators of \(\theta^{*}\). The variable ``Est. Variance'' estimates   \(\var_{H}(\psi)\)  via \eqref{eq:EmpVarPsiEst}, with \(K=k\) and \(n'=nk\).   The variable ``Cost'' (resp. ``Time'') represents the total number of simulated copies of \(x\) (resp. running time in seconds) needed to sample \(\psi^{(1)},\dots,\psi^{(n)}\). The variable ``Cost of Est.'' (resp. ``Time of Est.'') refers to  the total number of simulated copies of \(x\) (resp. running time in seconds) needed to estimate   \(\var_{H}(\psi)\)  via \eqref{eq:EmpVarPsiEst}.

The law of large numbers suggests that \(\bar\psi\) is ``close'' to \(E(\psi)\), which implies that  \(||\bar\psi-\theta^{*}||^{2}_{H}\) is close to the squared bias   \(||E(\psi)-\theta^{*}||^{2}_{H}\). When \(\psi=\bar\theta_{k}\), the values of   \(||\bar\psi-\theta^{*}||^{2}_{H}\) in Table \ref{tab:BiasVarBach} are indeed within \(0.5\%\) of the exact values of the squared bias   \(||E(\bar\theta_{k})-\theta^{*}||^{2}_{H}\) in Table \ref{tab:BiasBach}. When  \(\psi\) is equal to \(\hat f_{k}\) or  \(\hat f_{\text{avg},k}\), the values of   \(||\bar\psi-\theta^{*}||^{2}_{H}\) in Table \ref{tab:BiasVarBach} are of order \(10^{-5}\) or less, which is consistent with the equalities \(E(\hat f_{k})=E(\hat f_{\text{avg},k})=\theta^{*}\). For each of the three estimators of \(\theta^{*}\), the variance computed via \eqref{eq:EmpVarPsiEst}  is within \(20\%\) of the variance estimated via   \eqref{eq:EmpVarPsiFirst}. The variance of the three estimators decreases as \(k\) increases, which is consistent with Lemma \ref{le:BiasedlinReg} and Theorems \ref{th:BoundhHatfk} and \ref{th:BoundhHatfkAvg}. 
The variances of  \(\hat f_{k}\) and of \(\hat f_{\text{avg},k}\) are much larger than that of  \(\bar\theta_{k}\),  especially  for small values of \(k\). This can be explained intuitively by observing that the variances  of  \(\hat f_{k}\) and of \(\hat f_{\text{avg},k}\)  are partly due to the bias of  \(\bar\theta_{k}\), that is large when  \(k\) is small. 
For each value of \(k\),   the variable ``Cost'' is about twice for   the estimators \(\hat f_{k}\) and \(\hat f_{\text{avg},k}\)  than for \(\bar\theta_{k}\), which is consistent with our choice of \(q\). The variances of  \(\hat f_{k}\) and of \(\hat f_{\text{avg},k}\) are similar for all values of \(k\). To explain this, we note that, because \(\hat f_{\text{avg},k}\) is an averaged version of \(\hat f_{k}\), it should have a smaller  variance  if both estimators use the same value of \(q\).     On the other hand, the expected running time of \(Z_{\text{avg},k}\) is larger than that of \(Z_{k}\). The corresponding value of \(q\) is therefore smaller for  \(Z_{\text{avg},k}\), which causes  the variance of \(\hat f_{\text{avg},k}\) to increase. The two effects tend to cancel each other.
Finally,  the variable ``Time'' (resp. ``Time of Est.'') is roughly proportional to variable  ``Cost'' (resp. ``Cost of Est.'').
\begin{table}
\caption{Squared Bias and Variance  with \(H=\hbox{\rm diag}(i^{-3})\), \(1\le i\le 25\), \(\sigma^{2}=1\), \(s(k)=k/2\) and \(n=10^{8}/k\).}
\begin{tiny}
\begin{tabular}{rrrclclrrr}\hline
$k$  & Method  & $n$  & Sq. Dist.  &  Variance & Cost  & Est. Variance & Cost of Est. &Time & Time of Est.\\\hline
$50$ &AvSGD& $ 2\times10^6$ & $6.381\times10^{-3} $ &$0.082 $ &$9.80\times10^7 $ &$0.082 $ &$2.98\times10^8 $ &$84$ &$247$\\
$50$ &USGD& $ 2\times10^6$ & $3.003\times10^{-6} $ &$17 $ &$1.90\times10^8 $ &$19 $ &$3.90\times10^8 $ &$158$ &$331$\\
$50$ &AUSGD& $ 2\times10^6$ & $1.166\times10^{-5} $ &$19 $ &$1.91\times10^8 $ &$17 $ &$3.89\times10^8 $ &$166$ &$319$\\
$200$ &AvSGD& $ 5\times10^5$ & $3.726\times10^{-3} $ &$0.035 $ &$9.95\times10^7 $ &$0.035 $ &$3.00\times10^8 $ &$87$ &$243$\\
$200$ &USGD& $ 5\times10^5$ & $7.152\times10^{-6} $ &$2.5 $ &$1.90\times10^8 $ &$2.3 $ &$3.91\times10^8 $ &$162$ &$319$\\
$200$ &AUSGD& $ 5\times10^5$ & $7.126\times10^{-6} $ &$2.8 $ &$2.00\times10^8 $ &$2.3 $ &$3.90\times10^8 $ &$171$ &$321$\\
$800$ &AvSGD& $ 1.25\times10^5$ & $1.920\times10^{-3} $ &$0.014 $ &$9.99\times10^7 $ &$0.014 $ &$3.00\times10^8 $ &$86$ &$244$\\
$800$ &USGD& $ 1.25\times10^5$ & $1.715\times10^{-6} $ &$0.4 $ &$2.03\times10^8 $ &$0.37 $ &$3.85\times10^8 $ &$169$ &$316$\\
$800$ &AUSGD& $ 1.25\times10^5$ & $2.661\times10^{-6} $ &$0.38 $ &$1.88\times10^8 $ &$0.38 $ &$3.92\times10^8 $ &$159$ &$322$\\
$3200$ &AvSGD& $ 3.125\times10^4$ & $7.041\times10^{-4} $ &$0.0058 $ &$1.00\times10^8 $ &$0.0059 $ &$3.00\times10^8 $ &$86$ &$244$\\
$3200$ &USGD& $ 3.125\times10^4$ & $2.150\times10^{-6} $ &$0.05 $ &$1.87\times10^8 $ &$0.055 $ &$4.34\times10^8 $ &$156$ &$364$\\
$3200$ &AUSGD& $ 3.125\times10^4$ & $1.499\times10^{-6} $ &$0.054 $ &$1.93\times10^8 $ &$0.047 $ &$3.85\times10^8 $ &$164$ &$321$\\

\hline\end{tabular}
 \end{tiny}
\label{tab:BiasVarBach}
\end{table}

Fig.~\ref{fig:TimeVarianceBach} plots the product of the variance and average running time  of  \(\hat f_{k}\) (resp. \(\hat f_{\text{avg},k}\)), for several values of \(k\), calculated through the formula \(\text{Cost}\times\text{Variance}/n\), where the variables  ``Cost'' and  ``Variance'' are extracted from Table \ref{tab:BiasVarBach}. By the discussion at the end of Section \ref{sub:BiasVariance}, this product measures the efficiency of the corresponding estimator.  Fig.~\ref{fig:TimeVarianceBach} shows that this product decreases as \(k\) increases, which implies that  the efficiency of   \(\hat f_{k}\) (resp. \(\hat f_{\text{avg},k}\)) increases as \(k\) increases. This is in line with the discussion following Proposition \ref{pr:avghatf}.  
\begin{figure}
\begin{center}
\begin{tikzpicture}
\begin{loglogaxis}[
    title={},
    xlabel={$k$},
 ylabel={Avg. Cost$\times$Variance},
  legend pos=south west,
    ymajorgrids=true,
    grid style=dashed,
]
\addplot[
    color=blue,
    mark=square,
    ]
    coordinates {
(50 , 1627.1)
(200 , 959.226)
(800 , 643.028)
(3200 , 301.212)
    };
\addplot[
    color=green,
    mark=triangle,
    ]
    coordinates {
(50 , 1855.8)
(200 , 1116.57)
(800 , 571.265)
(3200 , 335.592)
};    
\legend{USGD, AUSGD}
\end{loglogaxis}
\end{tikzpicture}
\end{center}
\caption{Product of  Average Running Time and of Variance with \(H=\hbox{\rm diag}(i^{-3})\), \(1\le i\le 25\), \(\sigma^{2}=1\), \(s(k)=k/2\) and \(n=10^{8}/k\).}
\label{fig:TimeVarianceBach}
\end{figure}
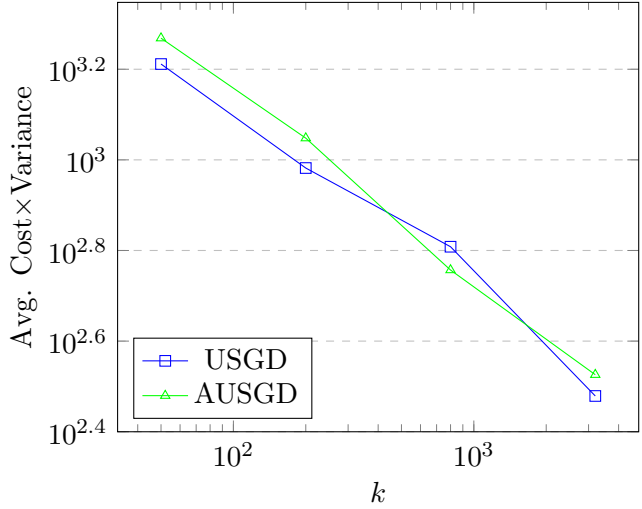

Fig. \ref{fig:ExcessRiskBach} plots the expected excess risk of \(\bar\psi_{M}\), calculated via \eqref{eq:ExpectedOptGapM}, where  \(\psi\) is equal to \(\bar\theta_{k}\), \(\hat f_{k}\), or  \(\hat f_{\text{avg},k}\), for several values of \(k\) and \(M\).  The exact value of the squared bias of  \(\bar\theta_{k}\) was  extracted from Table \ref{tab:BiasBach}, while the squared bias  of \(\hat f_{k}\) and of  \(\hat f_{\text{avg},k}\) was set to \(0\). The variance of the three estimators was extracted from the  variable ``Variance'' in Table \ref{tab:BiasVarBach}. Fig. \ref{fig:ExcessRiskBach} shows that the expected excess risk of \(\bar\psi_{M}\) is essentially the same when \(\psi\) equals \(\hat f_{k}\) or  \(\hat f_{\text{avg},k}\). For \(M=1\) or \(M=100\),    the smallest expected excess risk of \(\bar\psi_{M}\) occurs  when   \(\psi=\bar\theta_{k}\), in general, while the opposite happens when \(M=10^4\). This is explained by noting that, when  \(M\)   is very large, the bias term is dominant in \eqref{eq:ExpectedOptGapM}.      
\begin{figure}
\centering
\begin{subfigure}[b]{0.3\textwidth}
\centering
\begin{tikzpicture}[scale=.5]
\begin{loglogaxis}[
    title={$M = 1$},
    xlabel={$k$},
    ylabel={$E(L(\bar \psi_{M}))-L( \theta ^{*})$},
ymajorgrids=true,
    grid style=dashed,
]
 \addplot[
    color=red,
    mark=diamond,
    ]
    coordinates {
(50 , 0.0443199)
(200 , 0.0194352)
(800 , 0.00817922)
(3200 , 0.0032751)
};
\addplot[
    color=blue,
    mark=square,
    ]
    coordinates {
(50 , 8.56133)
(200 , 1.25948)
(800 , 0.19816)
(3200 , 0.0251964)
    };
\addplot[
    color=green,
    mark=triangle,
    ]
    coordinates {
(50 , 9.71937)
(200 , 1.39332)
(800 , 0.189771)
(3200 , 0.0271484)
};    
\legend{AvSGD, USGD, AUSGD}
\end{loglogaxis}
\end{tikzpicture}
\end{subfigure}
\begin{subfigure}[b]{0.3\textwidth}
\centering
\begin{tikzpicture}[scale=.5]
\begin{loglogaxis}[
    title={$M = 100$},
    xlabel={$k$},
    ylabel={$E(L(\bar \psi_{M}))-L( \theta ^{*})$},
    ymajorgrids=true,
    grid style=dashed,
]
 \addplot[
    color=red,
    mark=diamond,
    ]
    coordinates {
(50 , 0.00360226)
(200 , 0.00204021)
(800 , 0.00103529)
(3200 , 0.000382082)
};
\addplot[
    color=blue,
    mark=square,
    ]
    coordinates {
(50 , 0.0856133)
(200 , 0.0125948)
(800 , 0.0019816)
(3200 , 0.000251964)
    };
\addplot[
    color=green,
    mark=triangle,
    ]
    coordinates {
(50 , 0.0971937)
(200 , 0.0139332)
(800 , 0.00189771)
(3200 , 0.000271484)
};    
\legend{AvSGD, USGD, AUSGD}
\end{loglogaxis}
\end{tikzpicture}
\end{subfigure}
\begin{subfigure}[b]{0.3\textwidth}
\centering
\begin{tikzpicture}[scale=.5]
\begin{loglogaxis}[
    title={$M = 10,000$},
    xlabel={$k$},
    ylabel={$E(L(\bar \psi_{M}))-L( \theta ^{*})$},
    legend pos=south west,
    ymajorgrids=true,
    grid style=dashed,
]
 \addplot[
    color=red,
    mark=diamond,
    ]
    coordinates {
(50 , 0.00319508)
(200 , 0.00186626)
(800 , 0.000963854)
(3200 , 0.000353151)
};
\addplot[
    color=blue,
    mark=square,
    ]
    coordinates {
(50 , 0.000856133)
(200 , 0.000125948)
(800 , 1.9816e-005)
(3200 , 2.51964e-006)
    };
\addplot[
    color=green,
    mark=triangle,
    ]
    coordinates {
(50 , 0.000971937)
(200 , 0.000139332)
(800 , 1.89771e-005)
(3200 , 2.71484e-006)
};    
\legend{AvSGD, USGD, AUSGD}
\end{loglogaxis}
\end{tikzpicture}
\end{subfigure}
\caption{Expected Excess Risk with \(H=\hbox{\rm diag}(i^{-3})\), \(1\le i\le 25\), \(\sigma^{2}=1\), \(s(k)=k/2\) and \(n=10^{8}/k\).}
\label{fig:ExcessRiskBach}
\end{figure}
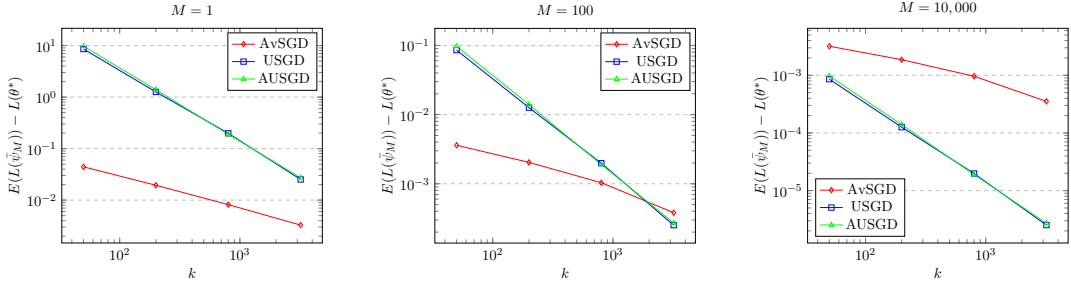

Table \ref{tab:BiasSidford} is constructed in the same way as Table \ref{tab:BiasBach}. For the two methods RMLMCB and URMLMCB and all values of \(k\) in Table \ref{tab:BiasSidford}, the estimated squared bias is within \(2\%\) from its exact value. Here again,  the variable  ``Cost'' is of order \(10^{8}\) or \(10^{9}\),
and the variable ``Time'' is roughly proportional to variable  ``Cost''.
\begin{table}
\caption{Squared Bias  of  averaged SGD with \(H=\hbox{\rm diag}(i^{-1})\), \(1\le i\le 50\), \(\sigma^{2}=0.01\), \(s(k)=k/2\) and \(n=10^{8}/k\).}
\begin{tabular}{rrcccc}\hline
$k$  & Method  & $n$  & Sq. Bias  & Cost  &Time\\\hline
$50$ &Exact& -- & $1.117\times10^{-1} $ &-- &--\\
$50$ &RMLMCB& $ 2\times10^6$ & $1.114\times10^{-1} $ &$6.18\times10^8 $ &$1.06\times10^3$\\
$50$ &URMLMCB& $ 2\times10^6$ & $1.116\times10^{-1} $ &$1.14\times10^9 $ &$1.94\times10^3$\\
$200$ &Exact& -- & $2.594\times10^{-2} $ &-- &--\\
$200$ &RMLMCB& $ 5\times10^5$ & $2.593\times10^{-2} $ &$6.28\times10^8 $ &$1.04\times10^3$\\
$200$ &URMLMCB& $ 5\times10^5$ & $2.593\times10^{-2} $ &$1.26\times10^9 $ &$2.08\times10^3$\\
$800$ &Exact& -- & $4.001\times10^{-4} $ &-- &--\\
$800$ &RMLMCB& $ 1.25\times10^5$ & $4.002\times10^{-4} $ &$6.22\times10^8 $ &$1.03\times10^3$\\
$800$ &URMLMCB& $ 1.25\times10^5$ & $4.000\times10^{-4} $ &$1.14\times10^9 $ &$1.89\times10^3$\\
$3200$ &Exact& -- & $5.096\times10^{-9} $ &-- &--\\
$3200$ &RMLMCB& $ 3.125\times10^4$ & $5.168\times10^{-9} $ &$6.08\times10^8 $ &$1.01\times10^3$\\
$3200$ &URMLMCB& $ 3.125\times10^4$ & $5.078\times10^{-9} $ &$1.16\times10^9 $ &$1.94\times10^3$\\

\hline\end{tabular}
\label{tab:BiasSidford}
\end{table}

Table \ref{tab:BiasVarSidford} is built in the same manner as Table \ref{tab:BiasVarBach}.
When \(\psi=\bar\theta_{k}\), the values of   \(||\bar\psi-\theta^{*}||^{2}_{H}\) in Table \ref{tab:BiasVarSidford} differ by less than \(0.2\%\) from the exact values of the squared bias  in Table \ref{tab:BiasSidford} for all values of \(k\) expect when \(k=3200\), in which case both quantities are of order \(10^{-8}\) or less.
When  \(\psi\) is equal to \(\hat f_{k}\) or  \(\hat f_{\text{avg},k}\), the values of   \(||\bar\psi-\theta^{*}||^{2}_{H}\) in Table \ref{tab:BiasVarSidford} are of order \(10^{-6}\) or less. For each of the three estimators of \(\theta^{*}\), the variance computed via  \eqref{eq:EmpVarPsiEst} is within \(5\%\) of the variance estimated via    \eqref{eq:EmpVarPsiFirst}. Once again, the variance of the three estimators decreases as \(k\) increases. For small values of \(k\), the variances of  \(\hat f_{k}\) and of \(\hat f_{\text{avg},k}\) are much larger than that of  \(\bar\theta_{k}\), but the three variances have a similar order of magnitude for large values of \(k\). This can be explained intuitively by the fast decay of the bias of  \(\bar\theta_{k}\), which is a source of the variance of  \(\hat f_{k}\) and of \(\hat f_{\text{avg},k}\). Here again, for each  of   the estimators \(\hat f_{k}\) and \(\hat f_{\text{avg},k}\), the variable ``Cost'' is about twice  that of \(\bar\theta_{k}\),  the variances of  \(\hat f_{k}\) and of \(\hat f_{\text{avg},k}\) are similar for all values of \(k\), and  the variable ``Time'' (resp. ``Time of Est.'') is roughly proportional to variable  ``Cost'' (resp. ``Cost of Est.'').

\begin{table}
\caption{Squared Bias and Variance  with \(H=\hbox{\rm diag}(i^{-1})\), \(1\le i\le 50\), \(\sigma^{2}=0.01\), \(s(k)=k/2\) and \(n=10^{8}/k\).}
\begin{tiny}
\begin{tabular}{rrrclclrrr}\hline
$k$  & Method  & $n$  & Sq. Dist. &  Variance & Cost  & Est. Variance & Cost of Est. &Time & Time of Est.\\\hline
$50$ &AvSGD& $ 2\times10^6$ & $1.116\times10^{-1} $ &$0.036 $ &$9.80\times10^7 $ &$0.036 $ &$2.98\times10^8 $ &$166$ &$514$\\
$50$ &USGD& $ 2\times10^6$ & $1.478\times10^{-6} $ &$3.9 $ &$1.90\times10^8 $ &$3.9 $ &$3.91\times10^8 $ &$320$ &$626$\\
$50$ &AUSGD& $ 2\times10^6$ & $3.139\times10^{-6} $ &$4.1 $ &$1.89\times10^8 $ &$4.1 $ &$3.99\times10^8 $ &$316$ &$646$\\
$200$ &AvSGD& $ 5\times10^5$ & $2.594\times10^{-2} $ &$0.015 $ &$9.95\times10^7 $ &$0.015 $ &$3.00\times10^8 $ &$168$ &$478$\\
$200$ &USGD& $ 5\times10^5$ & $1.067\times10^{-6} $ &$0.37 $ &$1.90\times10^8 $ &$0.37 $ &$3.88\times10^8 $ &$309$ &$621$\\
$200$ &AUSGD& $ 5\times10^5$ & $9.334\times10^{-7} $ &$0.37 $ &$1.93\times10^8 $ &$0.37 $ &$3.95\times10^8 $ &$322$ &$638$\\
$800$ &AvSGD& $ 1.25\times10^5$ & $3.996\times10^{-4} $ &$0.0017 $ &$9.99\times10^7 $ &$0.0017 $ &$3.00\times10^8 $ &$169$ &$480$\\
$800$ &USGD& $ 1.25\times10^5$ & $6.047\times10^{-8} $ &$0.0085 $ &$1.87\times10^8 $ &$0.0085 $ &$3.88\times10^8 $ &$306$ &$621$\\
$800$ &AUSGD& $ 1.25\times10^5$ & $1.121\times10^{-7} $ &$0.0082 $ &$1.91\times10^8 $ &$0.0082 $ &$3.91\times10^8 $ &$343$ &$632$\\
$3200$ &AvSGD& $ 3.125\times10^4$ & $2.005\times10^{-8} $ &$0.00045 $ &$1.00\times10^8 $ &$0.00043 $ &$3.00\times10^8 $ &$169$ &$480$\\
$3200$ &USGD& $ 3.125\times10^4$ & $1.569\times10^{-8} $ &$0.00045 $ &$1.92\times10^8 $ &$0.00043 $ &$3.90\times10^8 $ &$318$ &$688$\\
$3200$ &AUSGD& $ 3.125\times10^4$ & $1.363\times10^{-8} $ &$0.00045 $ &$1.85\times10^8 $ &$0.00047 $ &$3.88\times10^8 $ &$316$ &$629$\\
\hline\end{tabular}
 \end{tiny}
\label{tab:BiasVarSidford}
\end{table}
Fig. \ref{fig:TimeVarianceSidford} is built in the same way as Fig.~\ref{fig:TimeVarianceBach}.  Here again,   the efficiency of   \(\hat f_{k}\) (resp. \(\hat f_{\text{avg},k}\)) increases as \(k\) increases. 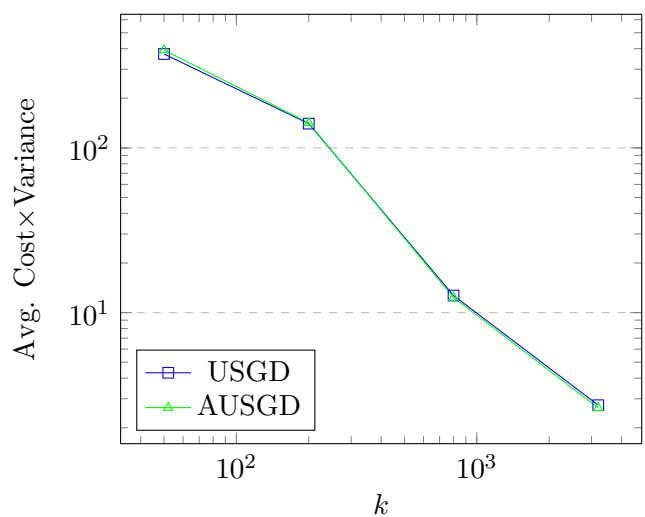
\begin{figure}
\begin{center}
\begin{tikzpicture}
\begin{loglogaxis}[
    title={},
    xlabel={$k$},
 ylabel={Avg. Cost$\times$Variance},
  legend pos=south west,
    ymajorgrids=true,
    grid style=dashed,
]
\addplot[
    color=blue,
    mark=square,
    ]
    coordinates {
(50 , 371.389)
(200 , 140.252)
(800 , 12.6958)
(3200 , 2.74194)
    };
\addplot[
    color=green,
    mark=triangle,
    ]
    coordinates {
(50 , 392.177)
(200 , 141.737)
(800 , 12.3552)
(3200 , 2.63982)
};    
\legend{USGD, AUSGD}
\end{loglogaxis}
\end{tikzpicture}
\end{center}
\caption{Product of  Average Running Time and of Variance with \(H=\hbox{\rm diag}(i^{-1})\), \(1\le i\le 50\), \(\sigma^{2}=0.01\), \(s(k)=k/2\) and \(n=10^{8}/k\).}
\label{fig:TimeVarianceSidford}
\end{figure}

Fig. \ref{fig:ExcessRiskSidford} is constructed in the same manner as Fig. \ref{fig:ExcessRiskBach}. One again,  the expected excess risk of \(\bar\psi_{M}\) is essentially the same when \(\psi\) equals \(\hat f_{k}\) or  \(\hat f_{\text{avg},k}\). Moreover, for \(M=1\) or \(M=100\),   the expected excess risk of \(\bar\psi_{M}\) is in general the smallest when  \(\psi=\bar\theta_{k}\), and the opposite happens when \(M=10^4\).  
\begin{figure}
\centering
\begin{subfigure}[b]{0.3\textwidth}
\centering
\begin{tikzpicture}[scale=.5]
\begin{loglogaxis}[
    title={$M = 1$},
    xlabel={$k$},
    ylabel={$E(L(\bar \psi_{M}))-L( \theta ^{*})$},
ymajorgrids=true,
    grid style=dashed,
]
 \addplot[
    color=red,
    mark=diamond,
    ]
    coordinates {
(50 , 0.0737852)
(200 , 0.0203559)
(800 , 0.00106433)
(3200 , 0.000223547)
};
\addplot[
    color=blue,
    mark=square,
    ]
    coordinates {
(50 , 1.95654)
(200 , 0.185012)
(800 , 0.00423576)
(3200 , 0.000223239)
    };
\addplot[
    color=green,
    mark=triangle,
    ]
    coordinates {
(50 , 2.06959)
(200 , 0.183364)
(800 , 0.00404257)
(3200 , 0.000223418)
};    
\legend{AvSGD, USGD, AUSGD}
\end{loglogaxis}
\end{tikzpicture}
\end{subfigure}
\begin{subfigure}[b]{0.3\textwidth}
\centering
\begin{tikzpicture}[scale=.5]
\begin{loglogaxis}[
    title={$M = 100$},
    xlabel={$k$},
    ylabel={$E(L(\bar \psi_{M}))-L( \theta ^{*})$},
    ymajorgrids=true,
    grid style=dashed,
]
 \addplot[
    color=red,
    mark=diamond,
    ]
    coordinates {
(50 , 0.0560051)
(200 , 0.0130447)
(800 , 0.000208671)
(3200 , 2.23799e-006)
};
\addplot[
    color=blue,
    mark=square,
    ]
    coordinates {
(50 , 0.0195654)
(200 , 0.00185012)
(800 , 4.23576e-005)
(3200 , 2.23239e-006)
    };
\addplot[
    color=green,
    mark=triangle,
    ]
    coordinates {
(50 , 0.0206959)
(200 , 0.00183364)
(800 , 4.04257e-005)
(3200 , 2.23418e-006)
};    
\legend{AvSGD, USGD, AUSGD}
\end{loglogaxis}
\end{tikzpicture}
\end{subfigure}
\begin{subfigure}[b]{0.3\textwidth}
\centering
\begin{tikzpicture}[scale=.5]
\begin{loglogaxis}[
    title={$M = 10,000$},
    xlabel={$k$},
    ylabel={$E(L(\bar \psi_{M}))-L( \theta ^{*})$},
    legend pos=south west,
    ymajorgrids=true,
    grid style=dashed,
]
 \addplot[
    color=red,
    mark=diamond,
    ]
    coordinates {
(50 , 0.0558273)
(200 , 0.0129716)
(800 , 0.000200114)
(3200 , 2.49025e-008)
};
\addplot[
    color=blue,
    mark=square,
    ]
    coordinates {
(50 , 0.000195654)
(200 , 1.85012e-005)
(800 , 4.23576e-007)
(3200 , 2.23239e-008)
    };
\addplot[
    color=green,
    mark=triangle,
    ]
    coordinates {
(50 , 0.000206959)
(200 , 1.83364e-005)
(800 , 4.04257e-007)
(3200 , 2.23418e-008)
};    
\legend{AvSGD, USGD, AUSGD}
\end{loglogaxis}
\end{tikzpicture}
\end{subfigure}
\caption{Expected Excess Risk with  \(H=\hbox{\rm diag}(i^{-1})\), \(1\le i\le 50\), \(\sigma^{2}=0.01\), \(s(k)=k/2\) and \(n=10^{8}/k\).}
\label{fig:ExcessRiskSidford}
\end{figure}
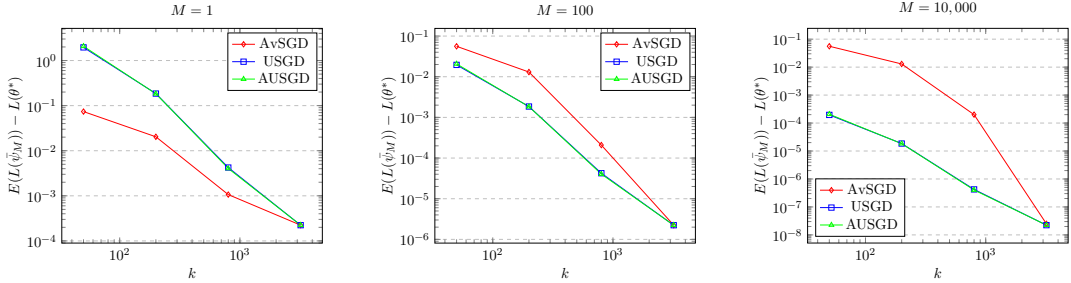
\subsection{Discussion}
Figs. \ref{fig:ExcessRiskBach} and \ref{fig:ExcessRiskSidford}   suggest that, for moderate values of \(M\),  it is more efficient to sample \(M\) independent copies of \(\bar\theta_{k}\) than  of \(\hat f_{k}\) or of \(\hat f_{\text{avg},k}\). But the contrary is true for very large values of \(M\) if the squared bias of  \(\hat f_{k}\) is non-negligible. Given \(k\ge2\) and \(\epsilon>0\), we can use \eqref{eq:ExpectedOptGapM} together with estimations of the squared bias  of \(\bar\theta_{k}\)   and of the variance of  \(\bar\theta_{k}\) and   of \(\hat f_{k}\) to determine an integer \(M\) with\begin{equation}\label{eq:target}
E(L(\bar\psi_{M}))-L( \theta ^{*})\le\epsilon^{2},
\end{equation}where \(\psi\) is either   \(\bar\theta_{k}\)  or \(\hat f_{k}\), without requiring knowledge of either \(H\) or  \(\theta^{*}\). Consider for instance the first distribution, and assume that  \(k=800\). The estimator URMLMCB in Table \ref{tab:BiasBach}  shows that \(||E(\bar\theta_{k})-\theta^{*}||^{2}_{H}\approx1.906\times10^{-3}\), while the variance estimator in Table \ref{tab:BiasVarBach} shows that \(\var_{H}(\bar\theta_{k})\approx0.014\) and \(\var_{H}(\hat f_{k})\approx0.37\). Using \eqref{eq:ExpectedOptGapM} and the equality \(E(\hat f_{k})=\theta^{*}\), it follows that  \(M=149\) (resp. \(M=185\)) independent copies of    \(\bar\theta_{k}\) (resp. \(\hat f_{k}\))  achieve \eqref{eq:target},    when \(\epsilon^{2}=10^{-3}\). However, when \(\epsilon^{2}=10^{-4}\), no integer \(M\)   achieves  \eqref{eq:target}    for the estimator     \(\bar\theta_{k}\), and   \eqref{eq:target}    is achieved    for  \(\hat f_{k}\)   when \(M=1850\). A similar calculation can be performed for \(\hat f_{\text{avg},k}\).  
\section{Conclusion}\label{se:conc}
We have studied the  least squares regression problem using the on-line framework of \citet{bachMoulines2013non}.   The expected number of time-steps of our unbiased estimator \(\hat f_{k}\)  of \(\theta^{*}\)   is of order \(k\), while the expected excess risk  of  \(\hat f_{k}\)  is \(O(1/k)\). The constant behind the \(O\) notation is, up to a poly-logarithmic factor, no greater  than the leading constant in  the bound on  the expected excess risk of \(\bar\theta_{k}\)  shown by  \citet{bachMoulines2013non}  for \(0<\gamma<1/R^{2}\). We   provide both a biased and  unbiased estimator of the squared bias    of  \(\bar\theta_{k}\),  and an unbiased estimator of the variance of  \(\bar\theta_{k}\) and of  \(\hat f_{k}\), that do not require knowledge of either \(H\) or  \(\theta^{*}\). This allows the numerical estimation of the expected excess risk of the average of \(M\) independent copies of \(\bar\theta_{k}\) and  of  \(\hat f_{k}\), for any integer \(M>1\). We also describe an average-start version of all our estimators, with similar properties. Our numerical experiments confirm our theoretical findings.
As a byproduct of our techniques, we prove an upper bound on the expected excess risk of  \(\bar\theta_{k}\) that is no worse than that  of \citet{bachMoulines2013non}, up to a constant factor, but with step-sizes twice as large. We provide an example where the expected excess risk of  \(\bar\theta_{k}\) goes to infinity with \(k\) for even larger step-sizes. Extensions of our research to accelerated optimization methods and to the minimization of other classes of functions is left for future work.
\section*{Acknowledgments} The author thanks Francis Bach and Aaron Sidford for helpful conversations.
\bibliography{poly}
\appendix
\section{Preliminary matrix inequalities}
\begin{lemma}
\label{le:traceDecrease}Let \(B\) be a symmetric positive semi-definite  \(d\times d\) matrix,   and let \(P\) be a random symmetric  \(d\times d\) matrix such that \(E(P^{2})\preccurlyeq I\). Then \(E(\tr(PBP))\leq\tr(B)\).
\end{lemma}
\begin{proof}
Since \(\tr(DC)=\tr(CD)\), 
\begin{eqnarray*}
E(\tr(PBP))
&=&E(\tr(\sqrt{B}P^{2}\sqrt{B}))\\
&=&\tr(\sqrt{B}E(P^{2})\sqrt{B})\\
&\le&\tr(B),
\end{eqnarray*}
where the last equation follows from the inequality \(\sqrt{B}E(P^{2})\sqrt{B}\preccurlyeq B\). 
\end{proof}
\begin{lemma}\label{le:matrixExpectationAndTrace}Let \((P_{t},t\ge0)\) be a random sequence of independent and identically distributed \(d\times d\) symmetric matrices such that \(E({P_{0}}^2)\preccurlyeq I-A\), where \(A\) is a symmetric positive semi-definite  \(d\times d\) matrix. For \(t\geq0\), let  \(M_{t}:=P_{{t-1}}P_{{t-2}}\cdots P_{{0}}\), with \(M_{0}:=I\). Then, for \(t\geq0\),  \begin{equation}\label{eq:upperBoundMatrixExpectations}
\sum^{t}_{j=0}E({M_{j}}^{T}AM_{j})\preccurlyeq I,
\end{equation}
 and\begin{equation}\label{eq:UpperBoundTrace}
E(\tr({M_{t}}^{T}AM_{t}))\leq \frac{d}{t+1}.
\end{equation} 
\end{lemma}
\begin{proof}
For \(j\geq0\), we have \(M_{j+1}=P_{j}M_{j}\). Thus\begin{eqnarray*}E({M_{j+1}}^{T}M_{j+1})&=&E({M_{j}}^{T}{P_{j}}^{2}M_{j})\\&\preccurlyeq&E({M_{j}}^{T}(I-A)M_{j})
\\&=&E({M_{j}}^{T}M_{j})- E({M_{j}}^{T}AM_{j}),\end{eqnarray*} where the second equation follows from the inequality \(E({P_{j}}^2)\preccurlyeq I-A\) and the independence of \(P_{j}\) and \(M_{j}\). 
Thus  \begin{displaymath}
 E({M_{j}}^{T}AM_{j})\preccurlyeq E({M_{j}}^{T}M_{j})-E({M_{j+1}}^{T}M_{j+1}).
\end{displaymath}Summing over \(j\) and noting that \({M_{t+1}}^{T}M_{t+1}\) is symmetric positive semi-definite implies \eqref{eq:upperBoundMatrixExpectations}.

We now prove~\eqref{eq:UpperBoundTrace}.  For \(t\geq0\), we have \begin{eqnarray*}
E(\tr(M_{t+1}A{M_{t+1}}^{T}))&=&E(\tr(P_{t}M_{t}A{M_{t}}^{T}P_{t}))\\
&\le&E(\tr(M_{t}A{M_{t}}^{T})),
\end{eqnarray*}
where the second equation follows from Lemma~\ref{le:traceDecrease} and the independence of \(P_{t}\) and \(M_{t}\). Thus the sequence  \((E(\tr(M_{t}A{M_{t}}^{T})),t\geq0)\) is decreasing. Because \(M_{t}\sim{M_{t}}^{T}\), we have \(E({M_{t}}^{T}AM_{t})=E(M_{t}A{M_{t}}^{T})\). Hence the sequence  \((E(\tr({M_{t}}^{T}AM_{t})),t\geq0)\) is decreasing as well.  Taking the trace in   \eqref{eq:upperBoundMatrixExpectations} implies~\eqref{eq:UpperBoundTrace}.
\end{proof}
\begin{lemma}\label{le:matrixExpectation}Let \((P_{t},t\ge0)\) be a random sequence of    \(d\times d\) symmetric matrices and let \((v_{t},t\ge0)\) be a random  sequence of   \(d\)-dimensional  column vectors.  Assume that the pairs    \((P_{t},v_{t}),t\ge0\), are independent, and that \(v_{t}\)  and each component of \(P_{t}\) is square-integrable. Suppose also that  \(E(v_{t})=0\) and \(E(v_{t}{v_{t}}^{T})\preccurlyeq \zeta(I-E({P_{t}}^2))\), where \(\zeta\) is a non-negative constant.   Define the random sequence  \((w_{t},t\ge0)\) of   \(d\)-dimensional  column vectors via the recursion \(w_{t+1}=P_{t}w_{t}+v_{t}\), \(t\geq0\), with \(w_{0}=0\). Then, for \(t\geq0\), \begin{equation}\label{eq:UpperBoundMatrix}
E(w_{t}{w_{t}}^T)\preccurlyeq \zeta I.
\end{equation} 
\end{lemma}
\begin{proof}
Because \(w_{t}\) is a deterministic function of     \((P_{j},v_{j}),0\le j\le t-1\),  it is  independent of \((P_{t},v_{t})\). Consequently, it can be shown by induction on \(t\) that \(E(w_{t})=0\) for \(t\geq0\). We now show 
\eqref{eq:UpperBoundMatrix}  by induction on \(t\). 
\eqref{eq:UpperBoundMatrix} clearly holds for \(t=0\). Assume that \eqref{eq:UpperBoundMatrix}  holds for \(t\). Since   \(w_{t}\) and \((P_{t},v_{t})\) are independent and  \(E(w_{t})=0\), we have \(E(P_{t}w_{t}{v_{t}}^T)=0\). As \(E((U+V)(U+V)^{T})=E(UU^{T})+E(VV^{T})\) for any random square-integrable  \(d\)-dimensional column vectors \(U\) and \(V\) with \(E(UV^{T})=0\), it follows that\begin{eqnarray*}E(w_{t+1}{w_{t+1}}^T)
&=&E(P_{t}w_{t}{w_{t}}^{T}P_{t})+E(v_{t}{v_{t}}^T)\\
&\preccurlyeq&\zeta E({P_{t}}^{2})+\zeta(I-E({P_{t}}^2))\\
&=&\zeta I,
\end{eqnarray*}
where the second equation follows from the induction hypothesis and the independence of \(w_{t}\) and \(P_{t}\). Thus  \eqref{eq:UpperBoundMatrix}  holds for \(t+1\).
\end{proof}
For any square-integrable  \(d\)-dimensional random column vectors \(U\) and \(V\), let \begin{equation*}
\cov(U,V):=E(U^T V)-E(U)^T E(V).
\end{equation*}Thus, \(\var(U)=\cov(U,U)\), and \(\cov(U,V)=0\) if \(U\) and \(V\) are independent.
For any random  square-integrable  \(d\)-dimensional column vectors \(U\),  \(V\) and \(V'\), any deterministic   \(d\times d\) matrix \(B\), and any random square-integrable  \(d\times d\)  matrix  \(B'\)  independent of \((U,V)\),  it can be shown that \(\cov(U,V+V')=\cov(U,V)+\cov(U,V')\), and   \(\cov(BU,V)=\cov(U,B^{T}V)\), with \(\cov(U,B'V)=\cov(U,E(B')V)\). Furthermore, if \(B\)    and \(B'\) are deterministic and symmetric  with \(B\preccurlyeq B'\), then \(\cov(U,BU)\le \cov(U,B'U)\). 
\begin{lemma}\label{le:rCov}Let \((P_{t},t\ge0)\) be a  random sequence of  identically distributed  \(d\times d\) matrices and let \((v_{t},t\ge0)\) be a random  sequence of   \(d\)-dimensional  column vectors. Assume that the pairs    \((P_{t},v_{t}),t\ge0\), are independent, and that \(v_{t}\)  and each component of \(P_{t}\) is square-integrable. Consider a random sequence  \((w_{t},t\ge0)\) of    \(d\)-dimensional  column vectors satisfying  the recursion \(w_{t+1}=P_{t}w_{t}+v_{t}\) for \(t\geq0\), where \(w_{0}\) is deterministic. Then, for non-negative integers \(t\) and \(j\) and any \(d\times d\) matrix \(B\), we have
\begin{equation}\label{eq:rcovEquality}
\cov(Bw_{t},w_{t+j})=\cov(Bw_{t},(E(P_{0}))^{j}w_{t}).
\end{equation}
\end{lemma}
\begin{proof}
As in the proof of Lemma \ref{le:matrixExpectation}, it can be shown that \(P_{t}\) is independent of \(w_{t}\), for \(t\geq0\). It follows  by induction on \(t\) that \(w_{t}\) is square-integrable, for \(t\geq0\). Thus, the two sides of  \eqref{eq:rcovEquality}  are well-defined. We prove \eqref{eq:rcovEquality} by induction on \(j\). The base case \(j=0\) is trivial. Assume now that  \eqref{eq:rcovEquality} holds for  \(j\) and any \(d\times d\) matrix \(B\). Because  \((w_{t},w_{t+j})\)    is a deterministic function of     \((P_{i},v_{i}),0\le i\le t+j-1\),  it is  independent of \((P_{t+j},v_{t+j})\). Thus,
\begin{eqnarray*}
\cov(Bw_{t},w_{t+j+1})&=&\cov(Bw_{t},P_{t+j}w_{t+j}+v_{t+j})\\
&=&\cov(Bw_{t},P_{t+j}w_{t+j})\\
&=&\cov(Bw_{t},E(P_{0})w_{t+j})\\
&=&\cov(E(P_{0})^{T}Bw_{t},w_{t+j})\\
&=&\cov(E(P_{0})^{T}Bw_{t},(E(P_{0}))^{j}w_{t})\\
&=&\cov(Bw_{t},(E(P_{0}))^{j+1}w_{t}),
\end{eqnarray*}
where the second equation follows from the independence of \(w_{t}\) and \(v_{t+j}\), the third one from the independence of \(P_{t+j}\) and \((w_{t},w_{t+j})\), and the fifth one from the induction hypothesis applied to the matrix \(E(P_{0})^{T}B\). Thus   \eqref{eq:rcovEquality} holds for  \(j+1\). \end{proof}
\begin{lemma}\label{le:sumMatrixInequality}
Let \(\Sigma \) be a \(d\times d\) symmetric matrix such that \(2I-\Sigma \) is positive-definite. Then, for \(i\geq0\), we have 
 \begin{equation*}
\sum^{i-1}_{t=0}(I-\Sigma )^{2t}\Sigma \preccurlyeq(2I-\Sigma )^{-1}.
\end{equation*}
\begin{proof}
It can be shown by induction on \(i\) that, for \(i\geq0\),
 \begin{equation*}
(2I-\Sigma )\Sigma \sum^{i-1}_{t=0}(I-\Sigma )^{2t}=I-(I-\Sigma )^{2i}.
\end{equation*}Consequently,\begin{eqnarray*}\sum^{i-1}_{t=0}(I-\Sigma )^{2t}\Sigma &=&(2I-\Sigma )^{-1}(I-(I-\Sigma )^{2i})\\
&=&(2I-\Sigma )^{-1}-(I-\Sigma )^{i}(2I-\Sigma )^{-1}(I-\Sigma )^{i}.\end{eqnarray*}
As \((2I-\Sigma )^{-1}\) is positive-definite, this concludes the proof.
\end{proof}
\end{lemma}
\section{Proof of Lemma~\ref{le:BiasedlinReg}}
\eqref{eq:upperBoundGammaLamdaMax} 
is an immediate consequence of  \eqref{eq:HR2I} and \eqref{eq:upperBoundGamma}. For \(t\geq0\), set\begin{equation}\label{eq:PkDef}
P_{t}:=I-\gamma x_{t}{x_{t}}^{T}.
\end{equation}By the definition of \(H\), we have\begin{equation}\label{eq:E(Pt)}
E(P_{t})= I - \gamma H,
\end{equation} for \(t\geq0\). Moreover, \eqref{eq:BasicSGD} implies that \begin{equation}\label{eq:deftheta}
\theta _{t+1}=P_{t}\theta _{t}+\gamma y_{t}x_{t},
\end{equation}for \(t\geq0\). Define the random sequence of \(d\)-dimensional column vectors \((\beta_{t},t\geq0)\) through the recursion \begin{equation}\label{eq:defBeta}
\beta _{t+1}=P_{t}\beta _{t}+\gamma y_{t}x_{t},
\end{equation}    with  \(\beta_{0}=\theta^{*}\). Thus \((\beta_{t},t\geq0)\) and \((\theta_{t},t\geq0)\) satisfy the same recursion. 
\subsection{Bounding the bias}
\begin{lemma}
For \(t\geq 0\), we have
\begin{equation}
\label{eq:Ethetak}E(\theta_{t}-\theta ^{*})=(I-\gamma H)^{t}(\theta_{0}-\theta ^{*}),
\end{equation}
and
\begin{equation}
\label{eq:Ebetak}E( \beta _{t})=\theta ^{*}.
\end{equation}
 
\end{lemma}
\begin{proof}
By \eqref{eq:BasicSGD}, \(\theta _{t}\) is  a deterministic function of \((x_{j},y_{j})\), \(0\leq j\leq t-1\). Hence  \(\theta _{t}\) is independent of \(P_{t}\). We prove~\eqref{eq:Ethetak} by induction on \(t\). Clearly, \eqref{eq:Ethetak} holds for \(t=0\). Assume that  \eqref{eq:Ethetak}  holds for \(t\). Then 
\begin{eqnarray*}
E(\theta _{t+1}-\theta ^{*})
&=&E(P_{t})E(\theta _{t})+\gamma E( yx)-\theta ^{*}\\
&=&(I -  \gamma H)E(\theta _{t})+\gamma H \theta^{*}-\theta ^{*}\\
&=&(I-\gamma H)E(\theta _{t}-\theta ^{*}),
\end{eqnarray*}where the first equation follows from  \eqref{eq:deftheta}  and the independence of   \(\theta _{t}\) and \(P_{t}\), and the second one from \eqref{eq:NormalEq}  and \eqref{eq:E(Pt)}. Thus \eqref{eq:Ethetak}  holds for \(t+1\).    
A similar inductive proof implies~\eqref{eq:Ebetak}.\end{proof} 
We now prove \eqref{eq:UpperBoundSqrBiasHBarTheta}. By~\eqref{eq:Ethetak}, for \(0\leq t\leq k-1\), 
\begin{equation*}\, ||E(\theta _{t}-\theta ^{*})||^{2}_{H}
={(\theta_{0}-\theta ^{*})}^{T}(I-\gamma H)^{2t}H(\theta_{0}-\theta ^{*}).\end{equation*}
Because the quadratic function \(v\mapsto ||v||^{2}_{H}\) is convex over \(\mathbb{R}^{d}\), and\begin{displaymath}
E(\bar \theta _{k}-\theta ^{*})=\frac{1}{k-s(k)}\sum^{k-1}_{t=s(k)}E(\theta _{t}-\theta ^{*}),
\end{displaymath} we have\begin{eqnarray*}||E(\bar \theta _{k}-\theta ^{*})||^{2}_{H}
&\le&\frac{1}{k-s(k)}\sum^{k-1}_{t=s(k)}||E(\theta _{t}-\theta ^{*})||^{2}_{H}
\\
&\le&\frac{1}{k-s(k)}\sum^{k-1}_{t=0}||E(\theta _{t}-\theta ^{*})||^{2}_{H}
\\
&=&\frac{1}{k-s(k)}\sum^{k-1}_{t=0}{(\theta_{0}-\theta ^{*})}^{T}(I-\gamma H)^{2t}H (\theta_{0}-\theta ^{*})\\
&\le&\frac{1}{\gamma( k-s(k))} {(\theta_{0}-\theta ^{*})}^{T}(2I-\gamma H)^{-1} (\theta_{0}-\theta ^{*})\\
&\le&\frac{||\theta_{0}-\theta ^{*}||^{2}}{\gamma(2-\gamma \lambda_{\max})( k-s(k))}.
\end{eqnarray*}
The fourth equation follows by applying Lemma~\ref{le:sumMatrixInequality} with \(\Sigma =\gamma H\) and observing that, by  \eqref{eq:upperBoundGammaLamdaMax}, the symmetric matrix  \(2I-\gamma H \) is positive-definite.
The last equation follows by noting that the largest eigenvalue of \((2I-\gamma H)^{-1}\) is \((2-\gamma \lambda_{\max})^{-1}\).
This implies \eqref{eq:UpperBoundSqrBiasHBarTheta}. 

If \(H\) is positive-definite then, owing to \eqref{eq:upperBoundGammaLamdaMax}, the eigenvalues of \(I-\gamma H\) are less then \(1\) in absolute value which, by \eqref{eq:Ethetak}, implies that \(E(\theta_{k})\rightarrow \theta^{*}\) as \(k\) goes to infinity.

\subsection{Bounding the variance}Set \begin{equation}\label{eq:defA}
A=\gamma(2-\gamma R^{2})H.
\end{equation}
Owing to \eqref{eq:upperBoundGamma}, the matrix \(A\) is symmetric positive semi-definite. \begin{lemma}\label{le:E(P0^2)}We have
\begin{equation}\label{eq:EP0^2}
E({P_{0}}^{2})\preccurlyeq I-A.
\end{equation}  For \(t\geq0\), we have\begin{equation}\label{eq:upperBoundMatrixYk}
E((\beta _{t}-\theta^{*})(\beta _{t}-\theta^{*})^{T})\leq\frac{\gamma\sigma^{2}}{2-\gamma R^{2}} I,
\end{equation}
and\begin{equation}\label{eq:upperBoundNormYk}
E(||\beta _{t}-\theta^{*}||^{2})\leq  \frac{\gamma\sigma^{2}d}{2-\gamma R^{2}} .
\end{equation}
Furthermore, \begin{equation}\label{eq:diffXkVk}
E(|| \theta_{t}-\beta _{t}||^{2})\leq||\theta_{0}-\theta^{*}||^{2}.
\end{equation}

\end{lemma}
\begin{proof}
We have
\begin{eqnarray*}
I - 2P_{0}+{P_{0}}^{2}&=&
(I -P_{0})^{2}\\&=&\gamma^{2}(x_{0}{x_{0}}^{T})^{2}\\
&=&\gamma^{2} ||x_{0}||^{2}{x_{0}}x_{0}^{T},
\end{eqnarray*}
where the second equation follows from \eqref{eq:PkDef}. Taking expectations and using \eqref{eq:E(Pt)} shows that \begin{eqnarray*}I -2(I - \gamma H )+E((P_{0})^{2})&=& \gamma^{2} E(||x||^{2}{x}x^{T})\\
&\preccurlyeq&\gamma^{2} R^{2}H,\end{eqnarray*}
where the second equation follows from \eqref{eq:Rdef}. This implies \eqref{eq:EP0^2}. 

We now prove \eqref{eq:upperBoundMatrixYk}. By \eqref{eq:defBeta}, for \(t\geq0\),
we have  \begin{equation*}
\beta _{t+1}-\theta^{*}=P_{t}(\beta _{t}-\theta^{*})+v_{t},
\end{equation*} where \(v_{t}=\gamma(y_{t}- x_{t}^T\theta^{*})x_{t}\).    For  \(t\geq0\), we have \(E(v_{t})=0\) because of~\eqref{eq:NormalEq}, and\begin{eqnarray*}
E(v_{t}{v_{t}}^{T})&=&\gamma^{2}E( (y- x^T\theta^{*})^{2}{x} x^{T})\\
&\preccurlyeq&\gamma^{2}\sigma^{2} H\\
&\preccurlyeq&\frac{\gamma\sigma^{2}}{2-\gamma R^{2}}(I-E({P_{0}}^{2})),
\end{eqnarray*}where the second equation follows from \eqref{eq:defSigma}, and the last one from \eqref{eq:EP0^2}. Applying Lemma~\ref{le:matrixExpectation} with \(w_{t}=\beta _{t}-\theta^{*}\) implies \eqref{eq:upperBoundMatrixYk}. As \(\tr(vv^T)=||v||^{2}\) for \(v\in\mathbb{R}^{d}\), taking the trace in  \eqref{eq:upperBoundMatrixYk} implies \eqref{eq:upperBoundNormYk}.

Next, we prove \eqref{eq:diffXkVk}.
By \eqref{eq:deftheta} and \eqref{eq:defBeta}, for \(t\geq0\), we have\begin{displaymath}
\theta_{t+1}-\beta _{t+1}=P_{t}(\theta_{t}-\beta _{t}).
\end{displaymath}Hence,  \(\theta_{t}-\beta _{t}\) is a deterministic function of \(P_{0},\dots,P_{t-1}\), and is independent of \(P_{t}\).   Thus,
\begin{eqnarray*}E(||\theta_{t+1}-\beta _{t+1}||^{2})
&=& E((\theta_{t}-\beta _{t})^{T}{P_{t}}^{2}(\theta_{t}-\beta _{t}))\\
&\le&E(|| \theta_{t}-\beta _{t}||^{2}),\end{eqnarray*}where the second equation follows from \eqref{eq:EP0^2} and the independence of \(\theta_{t}-\beta _{t}\) and \(P_{t}\). This implies \eqref{eq:diffXkVk}.
\end{proof}
\begin{lemma}\label{le:covSum} For \(0\leq j\leq h\), we have
\begin{displaymath}
\sum^{h}_{t=j}\cov(H\theta_{j},\theta_{t})\le4\left(\frac{\sigma^{2}d}{2-\gamma R^{2}}+ \frac{||\theta_{0}-\theta ^{*}||^{2}}{\gamma}\right).
\end{displaymath}
\end{lemma}
\begin{proof}
Using  \eqref{eq:E(Pt)} and \eqref{eq:deftheta} and applying Lemma~\ref{le:rCov} with \(w_{t}=\theta _{t}\) and \(v_{t}=\gamma y_{t}x_{t}\), we have
\begin{eqnarray*}
\sum^{h}_{t=j}\cov(H\theta _{j},\theta _{t})
&=&\sum^{h}_{t=j}\cov(H\theta _{j},(I-\gamma H)^{t-j}\theta _{j})\\
&=&\sum^{h}_{t=j}\cov(\theta _{j},H(I-\gamma H)^{t-j}\theta _{j})\\
&=&\gamma^{-1}\cov(\theta _{j},(I-(I-\gamma H)^{h+1-j})\theta _{j})\\
&\le&2\gamma^{-1}\var(\theta _{j}).
\end{eqnarray*}
The third equation follows  from the identity
 \begin{equation*}
\sum^{n}_{i=0}\Sigma(I-\Sigma )^{i} =I-(I-\Sigma )^{n+1},
\end{equation*} 
for \(n\ge0\) and any \(d\times d\) matrix \(\Sigma \). The last equation follows by observing that, by  \eqref{eq:HR2I},   \( -(I-\gamma H)^{n}\preccurlyeq I\), for \(n\geq0\). On the other hand, 
\begin{eqnarray*}
\var(\theta _{j}) 
&\le& 2(\var(\beta_{j})+\var(\theta _{j}-\beta_{j})) \\
&\le& 2\left(  \frac{\gamma\sigma^{2}d}{2-\gamma R^{2}} +||\theta_{0}-\theta ^{*}||^{2}\right),
\end{eqnarray*}
where the first equation follows from the inequality  \(\var(U+V)\le2(\var(U)+\var(V))\) for  square-integrable  \(d\)-dimensional column vectors \(U\) and  \(V\), and the second one from \eqref{eq:upperBoundNormYk} and  \eqref{eq:diffXkVk}. This concludes the proof.
\end{proof}
We now prove  \eqref{eq:UpperBoundVarHBarTheta}. Let  \(k\geq1\). We have \(\var_{H}(\bar \theta _{k})=\cov(H\bar \theta _{k},\bar \theta _{k})\). Because \(H\) is symmetric positive semi-definite,\begin{eqnarray*}
\cov(H\bar \theta _{k},\bar \theta _{k})&=&\frac{1}{(k-s(k))^{2}}(\sum^{k-1}_{j=s(k)}\cov(H \theta _{j}, \theta _{j})+2\sum^{k-1}_{j=s(k)}\sum^{k-1}_{t=j+1}\cov(H \theta _{j}, \theta _{t}))
\\&\le&\frac{2}{(k-s(k))^{2}}\sum^{k-1}_{j=s(k)}\sum^{k-1}_{t=j}\cov(H \theta _{j}, \theta _{t})
\\&\le&\frac{8}{k-s(k)}\left(\frac{\sigma^{2}d}{2-\gamma R^{2}}+ \frac{||\theta_{0}-\theta ^{*}||^{2}}{\gamma}\right),
\end{eqnarray*}where the last equation follows from Lemma~\ref{le:covSum}. This implies   \eqref{eq:UpperBoundVarHBarTheta}.
\subsection{Combining bias and variance terms} 
Combining \eqref{eq:HR2I} and \eqref{eq:UpperBoundSqrBiasHBarTheta} shows that, for \(k\geq1\),\begin{equation*}
||E(\bar \theta _{k}-\theta ^{*})||^{2}_{H}\leq\frac{||\theta_{0}-\theta ^{*}||^{2}}{\gamma(2-\gamma R^{2})( k-s(k))}.
\end{equation*} Together with \eqref{eq:BiasVarianceOptGap} and \eqref{eq:UpperBoundVarHBarTheta}, this implies \eqref{eq:thMainPositiveSemiDefinite}.

\section{Proof of Proposition~\ref{pr:OptimalRadius}}
  Because \(xx^{T}=DUU^{T}+(1-D)vv^{T}\) and \(E(UU^{T})=d^{-1}I\), we have \(H=(1-p)d^{-1}I +pvv^{T}\) by its definition. As \(0<p<1\), it follows that \(H\) is positive-definite. Moreover, because \(x\) and \(y\) are independent and \(E(y)=0\), we have \(E(yx)=0\). By \eqref{eq:NormalEq}, it follows that \(L(\theta)\) attains its minimum at \(\theta^{*}=0\). Moreover, using once again the independence of \(x\) and \(y\) shows that \(E(y^{2}{x}x^{T})=\sigma^{2}H\). Hence \eqref{eq:defSigma} holds. Finally, because \(||x||^{2}=1\) with probability \(1\), we have \(E(||x||^{2}{x}x^{T})=H\). Thus \eqref{eq:Rdef} holds with \(R^{2}=1\). On the other hand, a simple calculation shows that\ \((I-\gamma H)v=av\), where \(a=-(1+(\gamma-2)d^{-1})\). By a proof similar to that of \eqref{eq:Ethetak}, we have \(E(\theta_{k})=(I-\gamma H)^{k}\theta_{0}\) for \(k\geq0\). Because \((I-\gamma H)^{k}v=a^{k}v\), it follows that, for \(k\geq0\), \begin{equation*}
v^T E(\theta_{k})=a^{k}v^T \theta_{0}.
\end{equation*}
Hence, for \(k\geq1\), \begin{equation*}
v^T E(\bar\theta_{k})=\frac{1-a^{k}}{1-a}v^T \theta_{0}.
\end{equation*} Moreover, because  \(pvv^{T}\preccurlyeq H\), it follows from \eqref{eq:minimumf}  that,  for \(k\geq1\),  
\begin{equation*}
L(\bar\theta_{k})-L(0)\geq\frac{p}{2}(v^T \,\bar\theta_{k})^{2}.
\end{equation*}
Taking expectations and using Jensen's inequality implies that \begin{displaymath}
E(L(\bar\theta_{k}))-L(0)\geq\frac{p}{2}\left(v^T \,E(\bar\theta_{k})\right)^{2}.
\end{displaymath}
Consequently, \begin{displaymath}
E(L(\bar\theta_{k}))-L(0)\geq\frac{p}{2}\left(\frac{1-a^{k}}{1-a}\right)^{2}(v^T \theta_{0})^{2}.
\end{displaymath}
Because \(a<-1\), this implies that \(E(L(\bar \theta_{k}))\) goes to infinity as \(k\) goes to infinity. 
 
\section{Proof of   Lemma~\ref{le:upperBoundDiffY}}
Define \(P_{t}\), \(t\in\mathbb{Z}\), via  \eqref{eq:PkDef}, and \(A\) via \eqref{eq:defA}. For \(m\in\mathbb{Z}\) and \(t\geq m\), define  \(\beta_{m:t}\) by setting \(\beta_{m:m}=\theta^{*}\) and  
\begin{equation*}
 \beta_{m:t+1}=P_{t}\beta_{m:t} +\gamma y_{t}x_{t},
\end{equation*}
\(t\geq m\). Thus,  \((\beta_{m:t}, t\geq m)\) is a Markov chain starting at \(\theta^{*}\) and satisfying the same recursion as  \((\theta_{m:t}, t\geq m)\).
\begin{lemma}\label{le:YimYi}
For \(m\le0\le t\),
we have\begin{displaymath}
E(||\beta_{m:t}-\beta_{t}||_{H}^{2})\leq \frac{\sigma^{2}d}{( 2 -\gamma R^{2})^{2}(t+1)}.
\end{displaymath}\end{lemma}
\begin{proof}
For \(t\geq 0\), it can be shown by induction on \(t\) that  \(\beta_{m:t}-\beta_{t}=M_{t}(\beta_{m:0}-\theta^{*})\), where \(M_{t}\) is defined as in Lemma~\ref{le:matrixExpectationAndTrace}. Hence 
\begin{eqnarray*} E(||\beta_{m:t}-\beta_{t}||_{H}^{2}) &=&E(||M_{t}(\beta_{m:0}-\theta^{*})||_{H}^{2})\\
&=&E(\tr(\sqrt{H}M_{t}(\beta_{m:0}-\theta^{*})(\beta_{m:0}-\theta^{*})^{T}{M_{t}}^{T}\sqrt{H}))\\
&=&E(\tr(\sqrt{H}M_{t}E((\beta_{m:0}-\theta^{*})(\beta_{m:0}-\theta^{*})^{T}){M_{t}}^{T}\sqrt{H})).
\end{eqnarray*}
The second equation follows from the equality \(||v||_{H}^{2}=\tr(\sqrt{H}vv^{T}\sqrt{H})\), \(v\in\mathbb{R}^{d}\), whereas the last equation follows from the independence of \(M_{t}\) and \(\beta_{m:0}\). As \(\beta_{m:0}\sim\beta_{-m}\), it follows from \eqref{eq:upperBoundMatrixYk} that\begin{equation*}E((\beta_{m:0}-\theta^{*})(\beta_{m:0}-\theta^{*})^{T})\le  \frac{\gamma\sigma^{2}}{2 -\gamma R^{2}}I.\end{equation*}
Consequently, \begin{eqnarray*}
E(||\beta_{m:t}-\beta_{t}||_{H}^{2})
&\le&  \frac{\gamma\sigma^{2}}{2 -\gamma R^{2}}E(\tr(\sqrt{H}M_{t}{M_{t}}^{T}\sqrt{H}))\\
&=&\frac{\sigma^{2}}{( 2 -\gamma R^{2})^{2}}E(\tr({M_{t}}^{T}AM_{t})),\\&\le&  \frac{\sigma^{2}d}{( 2 -\gamma R^{2})^{2}(t+1)},
\end{eqnarray*}
where  the last equation follows from  \eqref{eq:EP0^2} and \eqref{eq:UpperBoundTrace}.
\end{proof}
\begin{lemma}\label{le:XiYiH}
Let \(J\) be a random integer uniformly distributed on \(\{h,\dots,h'\}\), where \(h\le h'\) are non-negative integers. Assume that \(J\) and \((x_{t},y_{t}),t\in \mathbb{Z}\), are independent. Then,\begin{equation*}
E(||\theta_{J}-\beta_{J}||^{2}_{H})\leq\frac{||\theta_{0}-\theta^{*}||^{2}}{\gamma(2 -\gamma R^{2})(h'-h+1)}.
\end{equation*}
\end{lemma}
\begin{proof}
It can be shown by induction on \(t\) that
\(\theta_{t}-\beta_{t}=M_{t}(\theta_{0}-\theta^{*})\) for \(t\geq0\). Thus, \begin{displaymath}
E(||\theta_{J}-\beta_{J}||^{2}_{H})=(\theta_{0}-\theta^{*})^{T}E({M_{J}}^{T}HM_{J})(\theta_{0}-\theta^{*}).
\end{displaymath}
On the other hand, \begin{eqnarray*}E({M_{J}}^{T}HM_{J})&=&\frac{1}{ \gamma(2 -\gamma R^{2})}E({M_{J}}^{T}AM_{J})\\&=&\frac{1}{\gamma(2 -\gamma R^{2})(h'-h+1)}\sum^{h'}_{j=h}E({M_{j}}^{T}AM_{j})\\&\le&\frac{1}{\gamma(2 -\gamma R^{2})(h'-h+1)}I,\end{eqnarray*}
where the last equation follows from \eqref{eq:upperBoundMatrixExpectations}.
This achieves the proof.\end{proof}
The proof of Proposition \ref{pr:simBeta} below follows immediately from that of \cite[Proposition B.1]{kahale2023}, and is omitted.  \begin{proposition}[\cite{kahale2023}]\label{pr:simBeta}For \(n,m,m'\in\mathbb{Z}\) with \(m, m'\le n\), we have \begin{equation}\label{eq:SimInd}
(\beta_{m:n},\beta_{m':n})\sim(\beta_{n-m},\beta_{m'-m:n-m}),
\end{equation}
and\begin{equation}\label{eq:SimIndTheta}
(\theta_{m:n},\theta_{m':n})\sim(\theta_{n-m},\theta_{m'-m:n-m}),
\end{equation}
\end{proposition}
We now prove Lemma~\ref{le:upperBoundDiffY}. 
For \(k\geq2\) and \(l\geq0\), we have
\begin{displaymath}
f_{k,l+1}-f_{k,l}=U+V +W,
\end{displaymath}
where \(U=\theta_{-\tau(k,l+1):0}-\beta_{-\tau(k,l+1):0}\), \(V=\beta_{-\tau(k,l+1):0} - \beta_{-\tau(k,l):0}\) and \(W=\beta_{-\tau(k,l):0}-\theta_{-\tau(k,l):0}\). Using the inequality \(||U+V+W||_{H}^{2}\le3(||U||_{H}^2+||V||_{H}^2+||W||_{H}^{2})\),  it follows that
\begin{equation}\label{eq:Xitaul+1}
||f_{k,l+1}-f_{k,l}||_{H}^{2}\le3(||U||_{H}^2+||V||_{H}^2+||W||_{H}^{2}). 
\end{equation}
Because \(\tau(k,l)\) and \(\tau(k,l+1)\) are independent of \((x_{t},y_{t})\), \(t\in\mathbb{Z}\),   \eqref{eq:SimInd} still holds if  \(m\) is replaced by \(-\tau( k,l)\) and \(m'\) by \(-\tau(k,l+1)\), with \(n=0\). Consequently, \begin{displaymath}
( \beta_{-\tau(k,l):0},\beta_{-\tau( k,l+1):0} )\sim(\beta_{\tau(k,l)},\beta_{\tau(k,l)-\tau( k,l+1):\tau(k,l)}  ),
\end{displaymath}  \(\)which implies that\begin{eqnarray}\label{eq:Yitaul+1}
E(||V||_{H}^{2})&=&E(||\beta_{\tau(k,l)-\tau(k,l+1):\tau(k,l)} - \beta_{\tau(k,l)}||_{H}^{2})\nonumber\\
&\le& \frac{\sigma^{2}d}{( 2 -\gamma R^{2})^{2}\max(s(k),k(2^{l}-1))},
\end{eqnarray}
where the second equation follows from Lemma~\ref{le:YimYi} and the inequalities \(\tau(k,l+1)\ge\tau(k,l)\ge \max(s(k),k(2^{l}-1))\).
In the same vein,     
it can be shown that
\((\theta_{-\tau(k,0):0},\beta_{-\tau(k,0):0})\sim(\theta_{\tau(k,0)},\beta_{\tau(k,0)})\)  which, by applying Lemma~\ref{le:XiYiH} with \(J=\tau(k,0)\) implies that
\begin{equation}\label{eq:Xitau0Bound}
E(||\theta_{-\tau(k,0):0}-\beta_{-\tau(k,0):0}||^{2}_{H})\leq\frac{||\theta_{0}-\theta^{*}||^{2}}{\gamma(2 -\gamma R^{2})(k-s(k))}.
\end{equation}
Similarly, for \(l\geq1\),
we have\begin{equation}\label{eq:XitaulBound}
E(||\theta_{-\tau(k,l):0}-\beta_{-\tau(k,l):0}||^{2}_{H})\leq\frac{||\theta_{0}-\theta^{*}||^{2}}{\gamma(2 -\gamma R^{2})k2^{l}}.
\end{equation}
Combining \eqref{eq:Xitaul+1}, \eqref{eq:Yitaul+1}, \eqref{eq:Xitau0Bound} and \eqref{eq:XitaulBound} shows after some calculations that  \(E(||f_{k,1}-f_{k,0}||_{H}^{2})\leq4 c/s(k)\). Similarly, combining \eqref{eq:Xitaul+1}, \eqref{eq:Yitaul+1} and \eqref{eq:XitaulBound} implies that \(E(||f_{k,l+1}-f_{k,l}||_{H}^{2})\leq 6c/(k2^{l})\) for \(l\geq1\).

\section{Proof of Lemma~\ref{le:upperBoundDiffYExp}}
\begin{lemma}\label{le:distanceFinitei}
We have \(\xi\leq1\) and, for \(m\le0\le t\),\begin{equation}\label{eq:distanceFinitei}
E(||\theta_{m:t}-\theta_{t}||^{2})\le  \frac{6c}{ R^{2}}(1- \xi)^{t}.
\end{equation}  
\end{lemma}
\begin{proof}
 We show \eqref{eq:distanceFinitei} by induction on \(t\). Since\begin{equation*}
\theta_{-m}-\theta_{0}=(\theta_{-m}-\beta_{-m})+(\beta_{-m}-\theta^{*})+(\theta^{*}-\theta_{0}),
\end{equation*}and \(||U+V+W||^{2}\le3(||U||^2+||V||^2+||W||^{2})\) for  \(U,V,W\in\mathbb{R}^{d}\),  it follows from \eqref{eq:upperBoundNormYk} and \eqref{eq:diffXkVk} that\begin{equation*}
E(||\theta_{-m}-\theta_{0}||^2)\le3\left(  \frac{\gamma\sigma^{2}d}{2 -\gamma R^{2}}+2||\theta_{0}-\theta^{*}||^{2}\right).
\end{equation*}Using the inequality \(\gamma R^{2} (2 -\gamma R^{2})\le1\), it follows after some calculations that \(E(||\theta_{-m}-\theta_{0}||^{2})\le {6c}/{R^{2}}\). Thus \eqref{eq:distanceFinitei} holds for \(t=0\).  Assume now that \eqref{eq:distanceFinitei} holds for \(t\). Combining the equation \(
 \theta_{m:t+1}=P_{t}\theta_{m:t} +\gamma y_{t}{x_{t}}
\) with \eqref{eq:deftheta}  implies that \(
\theta_{m:t+1}-\theta_{t+1}=P_{t}(\theta_{m:t}-\theta_{t}) 
\). Because \({P_{t}}\) is independent of \(\theta_{m:t}-\theta_{t}\), we conclude that\begin{equation*}E(||\theta_{m:t+1}-\theta_{t+1}||^{2})=E((\theta_{m:t}-\theta_{t})^{T}E({P_{t}}^{2})(\theta_{m:t}-\theta_{t})).
\end{equation*}
As \(\xi I\preccurlyeq  A\), it follows from~\eqref{eq:EP0^2} that \(E({P_{0}}^{2})\preccurlyeq(1-\xi)I
\).
As \({P_{0}}^{2}\) is positive semi-definite, this implies that    \(\xi\leq1\) and\begin{displaymath}
E(||\theta_{m:t+1}-\theta_{t+1}||^{2})\leq(1-\xi)E(||\theta_{m:t}-\theta_{t}||^{2}).
\end{displaymath}
Hence \eqref{eq:distanceFinitei} holds for \(t+1\).
\end{proof}

We now prove Lemma~\ref{le:upperBoundDiffYExp}. By \eqref{eq:HR2I}, we have  $||v||^2_{H}\le R^{2}||v||^{2}$ for \(v\in \mathbb{R}^{d}\).  By Lemma \ref{le:distanceFinitei}, it follows that \(E(||\theta_{m:t}-\theta_{t}||_{H}^{2})\le  6c(1- \xi)^{t}\) for \(m\le0\le t\). On the other hand, for \(0\le t\le j\), using     \eqref{eq:SimIndTheta} shows that  \((\theta_{-j:0} , \theta_{-t:0})\sim(\theta_{t-j:t} , \theta_{t})\). This implies that  \(E(||\theta_{-j:0} - \theta_{-t:0}||_{H}^{2})\le  6c(1- \xi)^{t}\). For \(k\geq2\) and \(l\geq0\), as \(\tau(k,l)\le\tau(k,l+1)\),  it follows that \begin{displaymath}
E(||f_{k,l+1} - f_{k,l}||_{H}^{2}|\tau(k,l),\tau(k,l+1))\le  6c(1- \xi)^{\tau(k,l)}.
\end{displaymath}Taking expectations and observing that \(\tau(k,l)\ge s(k)2^{l}\)  and  \(1-\xi\leq e^{-\xi}\) concludes the proof.  
\section{Proof of Lemma~\ref{le:ZkProperties}}
Fix \(k\geq2\). For \(l\geq0\), set \(Y_{l}=f_{k,l+1}-f_{k,0}\), with \(Y_{-1}=0\).  It follows from the discussion surrounding \eqref{eq:fklSimilarity} that \(Y_{l}\) is square-integrable for \(l\geq0\) and that \(E(Y_{l})\) converges to  \(\theta^{*}-E(\bar\theta_k)\) as \(l\) goes to infinity.  Moreover, we have    \(Y_{l}-Y_{l-1}=f_{k,l+1}-f_{k,l}\) for \(l\geq0\) and \(Z_{k}=(Y_{N}-Y_{N-1})/p_{N}\). We show below  that \((Y_{l},l\geq0)\) satisfies the conditions of Corollary \ref{cor:Glynn}.
\begin{lemma}\label{le:l0ToInfinity}Let \(l_{0}\) be a non-negative integer such that \begin{equation}\label{eq:k2^l0}
\frac{4\delta+8}{\alpha\xi}\ln\left(\frac{4\delta+8}{\alpha\xi}\right)\le k2^{l_{0}}.
\end{equation}Then\begin{equation}\label{eq:suml0Infty}
\sum^{\infty}_{l=l_{0}}\frac{E(||Y_{l}-Y_{l-1}||_{H}^{2})}{p_{l}}\leq \frac{24c}{k^{\delta+2}}\exp(-\alpha\xi   k/2).
\end{equation}
 \end{lemma}
 \begin{proof}
 Set \(\phi=(4\delta+8)/(\alpha\xi)\). Thus the left-hand side of \eqref{eq:k2^l0} is equal to \(\phi\ln(\phi)\). As  \(\alpha\le1/2\), and \(\delta>1\), and  \(\xi\leq1\) by Lemma~\ref{le:upperBoundDiffYExp}, we have \(\phi\ge 24\). Because \(u/\ln(u)\) is an increasing function of \(u\) on \([e,\infty)\), for \(u\geq \phi\ln(\phi)\), we have\begin{eqnarray*}
\frac{u}{\ln(u)}&\ge&\frac{\phi\ln(\phi)}{\ln(\phi\ln(\phi))}\\
&\ge&\frac{\phi}{2},
\end{eqnarray*}where the second equation follows from the inequality \(\ln(\phi)\le \phi\). After some calculations, it follows that \(u^{\delta+2}\leq e^{\alpha\xi  u/2}\) for \(u\geq \phi\ln(\phi)\). On the other hand, the condition \eqref {eq:k2^l0} shows that   \( k2^{l} \geq \phi\ln(\phi)\) for \(l\geq l_{0}\). Applying the preceding inequality with  \(u= k2^{l}\) shows that \begin{equation}\label{eq:kdeltap2}
k^{\delta+2}2^{l(\delta+2)}\leq \exp(\alpha \xi k2^{l}/2).
\end{equation} 
 \(\)Moreover, as \(l+1\leq2^{l}\), \eqref{eq:plbounds} implies that \begin{equation*}
\frac{1}{p_{l}}\leq2^{l(\delta+1)+1}.
\end{equation*}Using Lemma \ref{le:upperBoundDiffYExp},   it follows that 
\begin{eqnarray*}
\sum^{\infty}_{l=l_{0}}\frac{E(||Y_{l}-Y_{l-1}||_{H}^{2})}{p_{l}}
&\leq&12c\sum^{\infty}_{l=l_{0}}2^{l(\delta+1)}\exp(-\alpha\xi  k2^{l})\\
&\leq&\frac{12c}{k^{\delta+2}}\sum^{\infty}_{l=l_{0}}2^{-l}\exp(-\alpha\xi  k2^{l}/2)\\
&\leq&\frac{24c}{k^{\delta+2}}\exp(-\alpha\xi   k/2),\end{eqnarray*}where the second equation follows from \eqref{eq:kdeltap2}, and the  last one from the equality \(\sum^{\infty}_{l=0}2^{-l}=2\). \end{proof} 
We now prove Lemma~\ref{le:ZkProperties}.  
We first show \eqref{eq:ZNSecondMomentGeneralCase}. Define \(\phi\) as in the proof of Lemma~\ref{le:l0ToInfinity}, and let \(l_{0}\) be the smallest integer such that \(\phi^{2}\leq2^{l_{0}}\). As \(\phi\geq 24\), we have \(l_{0}>0\).   By Lemma~\ref{le:upperBoundDiffY} and  \eqref{eq:plbounds}, \begin{eqnarray*}\sum^{l_{0}-1}_{l=0}\frac{E(||Y_{l}-Y_{l-1}||_{H}^{2})}{p_{l}}
 &\leq&\frac{8c}{\alpha   k}+\frac{12c}{k}\sum^{l_{0}-1}_{l=1}(l+1)^{\delta}\\&\le&\frac{8c}{\alpha   k}+\frac{12c}{k(\delta+1)}(l_{0}+1)^{\delta+1},
\end{eqnarray*}
 where the second equation follows from the inequality \((\delta+1)l^{\delta}\leq(l+1)^{\delta+1}-l^{\delta}\)  for \(l\geq1\). Because \(l_{0}\) satisfies \eqref{eq:k2^l0}, Lemma~\ref{le:l0ToInfinity} shows that \eqref{eq:suml0Infty} is valid for \(l_{0}\). It follows that
\begin{eqnarray*}\sum^{\infty}_{l=0}\frac{E(||Y_{l}-Y_{l-1}||_{H}^{2})}{p_{l}}
 &\le&\frac{8c}{\alpha   k}+\frac{12c}{k(\delta+1)}(l_{0}+1)^{\delta+1}+\frac{24c}{k^{\delta+2}} \\&\le&\frac{11c}{\alpha   k}+\frac{6c}{k}(l_{0}+1)^{\delta+1},
\end{eqnarray*}
where the second equation follows from the inequalities \(\delta>1\), \(k\geq2\) and \(\alpha\leq1/2\).  On the other hand, \(2^{l_{0}-1}<\phi^{2}\) by the definition of \(l_{0}\). Hence \(l_{0}-1\le2\ln(\phi)/\ln(2)\). As \(2/\ln(2)\leq3\), it follows that \(l_{0}\le3\ln(\phi)+1\). Hence \begin{eqnarray*}
\sum^{\infty}_{l=0}\frac{E(||Y_{l}-Y_{l-1}||_{H}^{2})}{p_{l}}&\leq&\frac{11c}{\alpha   k}+\frac{6c}{k}(3\ln(\phi)+2)^{\delta+1}\\
  &\leq&\frac{11c}{\alpha   k}+\frac{6c}{k}(4\ln(\phi))^{\delta+1},
\end{eqnarray*}where the second equation follows from the inequalities \(2\leq\ln(24)\) and \(24\leq \phi\). Thus  the conditions of Corollary \ref{cor:Glynn} are satisfied for \((Y_{l},l\geq0)\). This implies that 
 \(Z_{k}\) is square-integrable and that \eqref{eq:ZNSecondMomentGeneralCase} holds. Furthermore, \(E(Z_{k})=  \theta^{*}-E(\bar\theta_k) \) since \(E(Y_{l})\rightarrow  \theta^{*}-E(\bar\theta_k) \) as \(l\) goes to infinity.

Assume now that \(k\geq \phi\ln(\phi)\). Then  \(l_{0}=0\) satisfies \eqref{eq:k2^l0}, and Lemma~\ref{le:l0ToInfinity} implies that \begin{equation*}
\sum^{\infty}_{l=0}\frac{E(||Y_{l}-Y_{l-1}||_{H}^{2})}{p_{l}}\leq \frac{24c}{k^{\delta+2}}\exp(-\alpha\xi   k/2).
\end{equation*}
Using Corollary~\ref{cor:Glynn}  implies \eqref{eq:kZkexponentialDec}.
\section{Proof of Theorem~\ref{th:BoundhHatfk}}
We first prove the upper bound on \(T_{k}\).
Given \(l\geq0\), it follows from \eqref{eq:recXim} and \eqref{eq:defFkl} that  \(f_{k,l+1}-f_{k,l}\) can be sampled by simulating \((x_{t},y_{t})\), \(-\tau(k,l+1)\le t\le -1\).   This takes \(\tau(k,l+1)\) time. As \(\tau(k,l)\le k2^{l+1}\) for \(l\geq 0\), given \(N\), the time to simulate \(Z_{k}\) is at most \(4k2^{N}\). Hence, for \(k\ge2\), we have\begin{eqnarray*}
T_{k}&\le&4k\sum^{\infty}_{l=0}2^{l}p_{l}\\
&\le&4k\sum^{\infty}_{l=0}\frac{1}{(l+1)^{\delta}}\\&\le&\frac{4\delta k}{\delta-1},
\end{eqnarray*}
where the second equation follows from \eqref{eq:plbounds} and the last one from the inequality \((l+1)^{-\delta}\leq\int^{l+1}_{l}u^{-\delta}\,du\) for  \(l\geq1\).
As the time to simulate \(\bar\theta_{k}\) is at most \(k\), the bound on \(\hat T_{k}\)  follows immediately from the definition of \(\hat f_{k}\).

Since \(Q\) and \(Z'_{k}\) are independent, we have \(E(QZ'_{k})=qE(Z'_{k})\). Using Lemma~\ref{le:ZkProperties},   
this implies that  \(E(\hat f_{k})=\theta^{*}\).  On the other hand, for \(k\ge2\), in view of the independence of \(\bar\theta_{k}\) and of \(QZ'_{k}\), we have \begin{displaymath}
\var_H (\hat f_{k})=\var_H (\bar\theta_{k})+\var_H (q^{-1}QZ'_{k}).
\end{displaymath}  By \eqref{eq:UpperBoundVarHBarTheta} and the definition of \(c\), we have \begin{eqnarray*}
\var_H (\bar\theta_{k})&\le&\frac{16c}{k-s(k)}\\&\le&\frac{16c}{\alpha k},
\end{eqnarray*}where the second equation follows from the inequality \(\alpha \leq 1/2\). Moreover,\begin{eqnarray*}\var_H (q^{-1}QZ'_{k})&\le&E(||q^{-1}QZ'_{k}||_{H}^{2})\\
&=&q^{-1}E(||Z_{k}||^{2}_{H})\\
&\le&\frac{1}{qk}\left(\frac{11c}{\alpha}+6c\left(4\ln\left(\frac{4\delta+8}{\alpha\xi}\right)\right)^{\delta+1}\right),
\end{eqnarray*}where the last equation follows from  \eqref{eq:ZNSecondMomentGeneralCase}.  Consequently, \begin{displaymath}
\var_H (\hat f_{k})\le\frac{1}{qk}\left(\frac{27c}{\alpha}+6c\left(4\ln\left(\frac{4\delta+8}{\alpha\xi}\right)\right)^{\delta+1}\right).
\end{displaymath}
Using \eqref{eq:BiasVarianceOptGap} and the equality \(E(\hat f_{k})=\theta^{*}\)  implies  \eqref{eq:optGapRMLMC}.
\section{Proof of Proposition \ref{pr:efficientRunningTime}}
We show by induction on \(t\) that, for \(m\leq t\leq0\), we have\begin{equation}\label{eq:recCalc}
\eta_{t}={\frac{1}{\min (m',t)+1-m}}\sum^{\min (m',t)}_{h=m}\theta_{h:t}.
\end{equation}
When \(t=m\), the two sides of \eqref{eq:recCalc}  are equal to \(\theta_{0}\), thus  \eqref{eq:recCalc} holds. Assume now that  \eqref{eq:recCalc} holds for \(t\), with  \(m\leq t\leq-1\). We shows that  \eqref{eq:recCalc} holds for \(t+1\) by  distinguishing the cases where  \(m\leq t\leq m'-1\) and where  \(m'\le t\le-1\). 

Suppose first that \(m\leq t\leq m'-1\). By \eqref{eq:etaDef}, \begin{equation}\label{eq:etaIplus1}
\eta_{t+1}=\frac{1}{t+2-m}((t+1-m)(P_{t}\eta _{t}+\gamma y_{t}{x_{t}})+\theta_{0}),
\end{equation}where \(P_{t}\) is defined via \eqref{eq:PkDef}. On the other hand, the induction hypothesis shows that \begin{eqnarray*}P_{t}\eta _{t}+\gamma y_{t}{x_{t}}&=&\frac{1}{t+1-m}P_{t}\left(\sum^{t}_{h=m}\theta_{h:t}\right)+\gamma y_{t}{x_{t}}
\\&=&\frac{1}{t+1-m}\sum^{t}_{h=m}(P_{t}\theta_{h:t}+\gamma y_{t}{x_{t}})\\
&=&\frac{1}{t+1-m}\sum^{t}_{h=m}\theta_{h:t+1},\end{eqnarray*} 
where the last equation follows from \eqref{eq:recXim}. Together with \eqref{eq:etaIplus1}, it follows that
 \begin{eqnarray*}\eta_{t+1}&=&\frac{1}{t+2-m}\left(\left(\sum^{t}_{h=m}\theta_{h:t+1}\right)+\theta_{0}\right)\\&=&\frac{1}{t+2-m}\sum^{t+1}_{h=m}\theta_{h:t+1}.\end{eqnarray*}
 Thus,  \eqref{eq:recCalc} holds \(t+1\). 

Assume now that \(m'\le t\le-1\). Then,   \begin{eqnarray*}\eta_{t+1}&=&P_{t}\eta _{t}+\gamma y_{t}{x_{t}}\\&=&\frac{1}{m'+1-m}P_{t}\left(\sum^{m'}_{h=m}\theta_{h:t}\right)+\gamma y_{t}{x_{t}}
\\&=&\frac{1}{m'+1-m}\sum^{m'}_{h=m}(P_{t}\theta_{h:t}+\gamma y_{t}{x_{t}})\\
&=&\frac{1}{m'+1-m}\sum^{m'}_{h=m}\theta_{h:t+1},\end{eqnarray*}
where the first equation follows from \eqref{eq:etaDef} and the second one from the induction hypothesis.
   Thus,  \eqref{eq:recCalc} holds for \(t+1\) in this case as well. Applying   \eqref{eq:recCalc} with \(t=0\) completes the proof.
\end{document}